\documentclass{article}

\usepackage{microtype}
\usepackage{graphicx}
\usepackage{subfigure}
\usepackage{booktabs} 

\usepackage{hyperref}

\usepackage[accepted]{icml2021}

\usepackage{amsmath,amsthm,amssymb}
\usepackage{xspace}
\usepackage[ruled, vlined, algo2e]{algorithm2e}
\usepackage{dsfont}
\usepackage{threeparttable}
\usepackage{subfigure}
\usepackage{graphicx}

\newcommand{\R}{\mathbb{R}}
\newcommand{\E}{\mathop{\mathbb{E}}}
\newcommand{\dtv}{d_{\text{TV}}}
\newcommand{\dentry}{d_{\text{ENTRY}}}
\newcommand{\Aone}{\mathcal{A}_1}
\newcommand{\Atwo}{\mathcal{A}_2}
\newcommand{\Athree}{\mathcal{A}_3}

\newcommand{\blue}[1]{\textcolor{black}{#1}}
\newcommand{\cameraReady}[1]{\textcolor{black}{#1}}

\newcommand{\newparagraph}[1]{\noindent\textbf{#1\xspace~~}}
\newcommand{\remark}[1]{\noindent\textit{#1\xspace~~}}

\newtheorem{theorem}{Theorem}
\newtheorem{lemma}{Lemma}
\newtheorem{corollary}{Corollary}
\newtheorem{definition}{Definition}
\newtheorem{proposition}{Proposition}

\DeclareMathOperator*{\argmax}{arg\,max}
\DeclareMathOperator*{\argmin}{arg\,min}
\DeclareMathOperator\supp{supp}

\icmltitlerunning{On Robust Mean Estimation under Coordinate-level Corruption}

\begin{document}

\twocolumn[
\icmltitle{On Robust Mean Estimation under Coordinate-level Corruption}



\icmlsetsymbol{equal}{*}

\begin{icmlauthorlist}
\icmlauthor{Zifan Liu}{equal,wisc}
\icmlauthor{Jongho Park}{equal,wisc}
\icmlauthor{Theodoros Rekatsinas}{wisc}
\icmlauthor{Christos Tzamos}{wisc}
\end{icmlauthorlist}

\icmlaffiliation{wisc}{Department of Computer Science, University of Wisconsin-Madison, Madison, USA}

\icmlcorrespondingauthor{Zifan Liu}{zliu676@wisc.edu}
\icmlcorrespondingauthor{Jongho Park}{jongho.park@wisc.edu}
\icmlkeywords{Machine Learning, ICML}

\vskip 0.3in
]



\printAffiliationsAndNotice{\icmlEqualContribution} 

\begin{abstract}
We study the problem of robust mean estimation and introduce a novel Hamming distance-based measure of distribution shift for coordinate-level corruptions. We show that this measure yields adversary models that capture more realistic corruptions than those used in prior works, and present an information-theoretic analysis of robust mean estimation in these settings. We show that for structured distributions, methods that leverage the structure yield information theoretically more accurate mean estimation. We also focus on practical algorithms for robust mean estimation and study when data cleaning-inspired approaches that first fix corruptions in the input data and then perform robust mean estimation can match the information theoretic bounds of our analysis.
We finally demonstrate experimentally that this two-step approach outperforms structure-agnostic robust estimation and provides accurate mean estimation even for high-magnitude corruption.
\end{abstract}

\section{Introduction}\label{sec:intro}
Data corruption is an impediment to modern machine learning deployments. Corrupted data samples, i.e., data vectors with either noisy or missing values, can severely skew the statistical properties of a data set, and hence, lead to invalid inferences. Robust statistics seek to provide methods for problems such as estimating the mean and covariance of a data set that are resistant to data corruptions.

Much of the existing robust estimation methods assume that a data sample is either completely clean or completely corrupted; \cameraReady{{\em Huber contamination model}~\cite{huber1992robust} and the {\em strong contamination model} considered by \citet{diakonikolas2019recent} are typical examples of such corruption models.} Under this kind of model, robust estimators rely either on filtering or down-weighting corrupted data vectors to reduce their influence~\cite{diakonikolas2019robust, 10.5555/3305381.3305485}.
In many applications, however, we can have partially corrupted data samples and even all data vectors can be partially corrupted. 
For example, in DNA microarrays, measurement errors or dropouts can occur for batches of genes~\cite{troyanskaya2001missing}. Filtering or down-weighting an entire data vector can waste the information contained in the clean coordinates of the vector. 

Moreover, recent works~\cite{DBLP:journals/pvldb/RekatsinasCIR17, wu2020att, ijcai2019-377} show that to obtain state-of-the-art empirical results for predictive tasks over noisy data, one needs to leverage the redundancy in data samples introduced by statistical dependencies among the coordinates of the data (referred to as {\em structure} hereafter) to learn an accurate distribution of the clean data and use that to repair corruptions. This work aims to promote theoretical understanding as to why leveraging statistical dependencies is key in dealing with data corruption. To this end, we study the connections between robust statistical estimation under worst-case (e.g., adversarial) corruptions and {\em structure-aware} recovery of corrupted data. 


\newparagraph{Problem Summary}
We consider \emph{robust mean estimation} under \emph{coordinate-level} corruptions (either missing entries or value replacements). We study worst-case, adversarial corruption, i.e., we assume that corruption is systematic and cannot be modeled as random noise. 
We consider the \emph{adversarial model} for which a given data set generated from an unknown distribution can have up to $\alpha$-fraction of its coordinates corrupted adversarially, i.e., the adversary can strategically hide or modify individual coordinates of samples. The goal is to find an estimate $\hat{\mu}$ of the true mean $\mu$ of the data set that is accurate even in the worst case. 



\newparagraph{Main Contributions}
 First, we present a new information theoretic analysis of robust mean estimation for coordinate-level corruptions, i.e., both replacement-based corruptions and missing-data corruptions. Our analysis introduces a model for coordinate-level corruptions, and $\dentry$, a new measure of distribution shift for coordinate-level corruption. We base $\dentry$ on the Hamming distance between samples from the original and the corrupted distribution. The reason is that the Hamming distance between the original and the corrupted samples is at most the amount of corruption in this sample, neatly reflecting the noise level. 


We present an information-theoretic analysis of corruption under coordinate-level adversaries and formally validate the empirical observations in the data cleaning literature: one must exploit the structure of the data to achieve information-theoretically optimal error for mean estimation. To show that structure is key, we focus on the case where the data lies on a lower-dimensional subspace, i.e., the observed sample before corruption is $x = Az$, where $A \in \R^{n\times r}$ and $z \in \R^r$ is a lower-dimensional vector drawn from an unknown distribution $D_z$. Also, $n$ is the total number of coordinates in $x$. Such low-rank subspace-structure is common in real-world data and the linear assumption is standard in theoretical exploration. We identify a key quantity $m_A$, the minimum number of rows that one needs to remove from $A$ to reduce its row space by one, which captures the effect of structure on mean-estimation error $||\mu-\hat{\mu}||$. For Gaussian distributions, the de facto distribution considered in the robust statistics literature, we prove that no algorithm can achieve error better than $\Omega(\alpha \frac{n}{m_A})$ when $\alpha$-fraction of coordinates per sample on average is adversarially corrupted. Our analysis highlights that, for coordinate-level corruption, it is necessary to use the structure in the data to perform recovery before statistical estimation. Specifically, when $\alpha$-fraction of the values are corrupted, recently-introduced estimators~\cite{diakonikolas2019robust} yield an estimation error of $\Omega(\alpha n)$ which is not the information theoretic optimal in the presence of structure. We show that to achieve the information theoretic optimal error of $\Theta(\alpha \frac{n}{m_A})$, one needs to consider the dependencies amongst coordinates.

Second, we study the existence of practical algorithms with polynomial complexity which yield results that match the error bounds of our information-theoretic analysis. We first show that, when corruptions correspond to missing data, a data-cleaning-inspired two-step approach achieves the information-theoretic optimal error. Specifically, to achieve the optimal error, one must first use imputation strategies that leverage the dependencies between data coordinates to recover the missing entries and then proceed with robust mean estimation over the imputed data. We show that in the case of linear structure, if the dependencies across coordinates are modeled via a known matrix $A$, we can recover missing entries by solving a linear system; when $A$ is unknown, we show that under bounded amount of corruption, we can leverage matrix completion methods~\cite{matrixCompletion} to recover the missing entries and obtain the same performance as in the case with known structure. We also explore replacement-type corruptions. By drawing connections to sparse recovery, we show that recovery for replacements is computationally intractable in general. As a preliminary result, we propose a randomized recovery algorithm for replacements that achieves probabilistic guarantees when the structure $A$ is known.

Finally, we present an experimental evaluation of the aforementioned two-step approach for missing-data corruptions on real-world data and show it leads to significant accuracy improvements over prior robust estimators even for samples that do not follow a Gaussian distribution or whose structure does not conform to a linear model.

\section{Background and Motivation}\label{sec:background}
We review the problem of robust mean estimation and discuss models and measures related to our study.

\newparagraph{Robust Mean Estimation} Robust mean estimation seeks to recover the mean $\mu \in \mathbb{R}^n$ of a $n$-dimensional distribution $D$ from a list of i.i.d. samples where an unknown number of arbitrary corruptions has been introduced in the samples. 
Given access to a collection of $N$ samples $x_1, x_2, \dots, x_N$ from $D$ on $\mathbb{R}^n$ when a fraction of them have been fully or partially corrupted, robust mean estimation seeks to find a vector $\hat{\mu}$ such that $\|\mu - \hat{\mu}\|$ is as small as possible. We consider two norms to measure the mean estimation error. The first norm is the Euclidean ($\ell_2$) distance and the second is the scale-invariant \emph{Mahalanobis distance} defined as $\|\mu - \hat{\mu}\|_{\Sigma} = |(\mu-\hat{\mu})^T\Sigma^{-1}(\mu-\hat{\mu})|^{1/2}$, where $\Sigma$ is the covariance matrix. When the covariance matrix is the identity matrix, the Mahalanobis distance reduces to the Euclidean distance.

\newparagraph{Sample-level Corruption}A typical model to describe worst-case corruptions is that of a \emph{sample-level adversary}, hereafter denoted $\Aone^\epsilon$. \cameraReady{In this paper, we assume that this adversary is allowed to inspect the samples and corrupt an $\epsilon$-fraction of them in an arbitrary manner. All coordinates of those samples are considered corrupted.}
Corruptions introduce a shift of the distribution $D$, which we can measure using the \emph{total variation distance} ($\dtv$). Total variation distance between two distributions $P$ and $Q$ on $\R^n$ is defined as $\dtv(P, Q) = \sup_{E \subseteq \R^n} |P(E) - Q(E)|$ or equivalently $\frac{1}{2}\|P-Q\|_1$. 
For two Gaussians $D_1 = \mathcal{N}(\mu_1, \Sigma)$ and $D_2 = \mathcal{N}(\mu_2, \Sigma)$ with $\dtv(D_1,D_2) = \epsilon < 1/2$ it is that $\|\mu_1-\mu_2\|_\Sigma = \Theta(\epsilon)$, i.e., their total variation distance and the Mahalanobis distance of their means are equivalent up to constants. This result allows tight analyses of Gaussian mean estimation for a bounded fraction of corruptions.

\newparagraph{Motivation} Total variation only provides a coarse measure of distribution shift, and hence, leads to a coarse characterization of the mean estimation error. For example, corruption of one coordinate per sample versus corruption of all coordinates results in the same distribution shift under total variation. 
However, the optimal mean estimation error can be different for these two cases. For example, if a corrupted coordinate has identical copies in other uncorrupted coordinates, the effect to mean estimation should be zero as we can repair the corrupted coordinate. This motivates our study.

\section{Information Theoretic Analysis}\label{sec:mean_estimation}
We study robust mean estimation under fine-grained corruption schemes. First, we introduce two new coordinate-level corruption adversaries (models) and a new measure of distribution shift ($\dentry$) that characterizes the effect of those adversaries on the observed distribution. Second, we present an information theoretic analysis of robust mean estimation and prove information-theoretically optimal bounds for mean estimation over Gaussians $\mathcal{N}(\mu, \Sigma)$ under coordinate-level corruption with respect to Mahalanobis distance. The results in this section hold for replacement-based corruptions as well as missing values. All proofs are deferred to the supplementary material of our paper.

\subsection{Coordinate-level Corruption Adversaries}\label{sec:corruptionModels}
We introduce two new adversaries and compare them to the sample-level adversary $\Aone^\epsilon$ from Section~\ref{sec:background}:

First, we consider an extension of $\Aone^\epsilon$ to coordinates, and define a \emph{value-fraction adversary}, denoted by $\Atwo^\rho$. Given $N$ samples from distribution $D$ on $\mathbb{R}^n$, adversary $\Atwo^\rho$ is allowed to corrupt up to a $\rho$-fraction of values in each coordinate of the $N$ samples. This adversary can corrupt a total of $\rho \cdot N \cdot n$ values in the $N$ samples; these values can be distributed strategically across samples leading to cases where \emph{most} of the samples are corrupted but still the corruption per coordinate is bounded by $\rho N$. 

Second, we define the more powerful \emph{coordinate-fraction adversary} $\Athree^\alpha$ that can corrupt \emph{all samples} in the worst case. $\Athree^\alpha$ is allowed to corrupt up to $\alpha$-fraction of all values in the $N$ samples, i.e., up to a total of $\alpha \cdot N \cdot n$ values. When $\alpha \geq \frac{1}{n}$, adversary $\Athree^\alpha$ can corrupt \emph{all} $N$ samples.

\cameraReady{Note that similar to $\Aone^\epsilon$, the coordinate-level adversaries we consider ($\Atwo^\rho$, and $\Athree^\alpha$) are adaptive, i.e., they are allowed to inspect the samples before choosing a fraction of the coordinates to corrupt.} 

\newparagraph{Adversary Comparison} $\Aone^\epsilon$ corresponds to the standard adversary associated with the strong contamination model considered by \citet{diakonikolas2019recent}, which either corrupts a sample completely or leaves it intact.
Adversaries $\Atwo^\rho$ and $\Athree^\alpha$ are more fine-grained since they can corrupt only part of the entries of a sample. As a result, $\Atwo^\rho$ and $\Athree^\alpha$ can corrupt more samples than $\Aone^\epsilon$ for similar budget-fractions $\epsilon$, $\rho$, and $\alpha$. 

We formalize the comparison among $\Aone^\epsilon$, $\Atwo^\rho$, and $\Athree^\alpha$ in the next propositions. 
We seek to understand when an adversary $\mathcal{A}$ can \emph{simulate} another $\mathcal{A}^\prime$, i.e., $\mathcal{A}$ can perform any corruption performed by $\mathcal{A}^\prime$.
\begin{proposition}
\label{lowerBoundAdversaryProposition}
If $\alpha,\, \rho \leq \epsilon / n$, then $\Aone^\epsilon$ can simulate $\Atwo^\rho$ and $\Athree^\alpha$. If $\alpha \le \rho / n$, $\Atwo^\rho$ can simulate $\Athree^\alpha$.
\end{proposition}
\begin{proposition}
\label{adversaryLemma}
If $\alpha,\, \rho \geq \epsilon$, then $\Atwo^\rho$ and $\Athree^\alpha$ can simulate $\Aone^\epsilon$. If $\alpha \ge \rho$, $\Athree^\alpha$ can simulate $\Atwo^\rho$. 
\end{proposition}
These propositions show that the two adversary types (sample- and coordinate-level) can simulate each other under different budget conditions, thus, enabling reductions between the two types.

Proposition~\ref{lowerBoundAdversaryProposition} implies that we can reduce coordinate-level corruption to sample-level corruption by considering $\Athree^\alpha$ as $\Aone^\epsilon$ with $\epsilon = \alpha n$. This reduction guarantees that \emph{any algorithm for mean estimation with guarantees for $\Aone^\epsilon$ enjoys the same guarantees for coordinate-level corruption when $\epsilon \geq \alpha \cdot n$.} Similarly, Proposition~\ref{adversaryLemma} means that \emph{any lower-bound guarantee on mean estimation for $\Aone^\epsilon$ also holds for $\Athree^\alpha$ when $\epsilon = \alpha$}. However, this characterization is loose as the gap between $\alpha$ and $\alpha n$ is large, raising the question: Are there distributions for which this gap is more tight and are there data properties we can exploit to reduce the dimensional factor of $n$? Next, we show that structure in data affects the power of coordinate-level corruption and introduces information-theoretically tight bounds for mean estimation under coordinate-level corruption.

\subsection{Distribution Shift in Coordinate-level Corruption}\label{sec:dentry}

We propose a new type of distribution shift metric, referred to as $\dentry$, which can capture fine-grained coordinate-level corruption. We have the next definition:

\begin{definition}[$\dentry$]\label{def:dentry}
Consider the coupling $\gamma$ of two distributions $P,Q$, i.e., a joint distribution of $P$ and $Q$ such that the marginal distributions are $P, Q$. Let the set of all couplings of $P,Q$ be $\Gamma(P, Q)$, and define for $x,y \in \R^n$, $I(x,y) = [\mathds{1}_{x_1\neq y_1}, \dots, \mathds{1}_{x_n\neq y_n}]^\top$. For $D_1, D_2$ on $\R^n$,
\begin{align*}
\dentry^1(D_1, D_2) &= \inf_{\gamma \in \Gamma(D_1, D_2)} 
\frac{1}{n}\|\E_{(x,y)\sim \gamma}\left[I(x,y)\right]\|_1 \\
\dentry^\infty (D_1, D_2) &=  \inf_{\gamma \in \Gamma(D_1, D_2)} \|\E_{(x,y)\sim \gamma}\left[I(x,y)\right] \|_\infty
\end{align*}
\end{definition}

\cameraReady{The following theorem shows the relation between $\dentry^{1}$ and $\Athree^{\alpha}$.}
\begin{theorem}\label{thm:relation}
\cameraReady{Let $D_1, D_2$ be two distributions such that $\dentry^1(D_1, D_2) = \alpha'$. $\Athree^{\alpha}$ corrupts $\alpha$ fraction of $N$ samples from $D_1$. If $\alpha > 2 \alpha'$, $\Athree^{\alpha}$ has a way to make corruptions so that with probability at least $1-e^{-\Omega(\alpha^2 N)}$ it is indistinguishable whether the $N$ samples come from $D_1$ or $D_2$. If $\alpha < \alpha'/4$, no matter how $\Athree^{\alpha}$ makes corruptions, with probability at least $1-e^{-\Omega(\alpha^2 N)}$, we can tell that the $N$ samples come from $D_1$.}
\end{theorem}

\cameraReady{The relation above also holds for $\dentry^{\infty}$ and $\Atwo^\rho$. The theorem shows that $\dentry$ gives a tight asymptotic bound on the power of coordinate-level adversaries. }

Intuitively $\dentry^1 (D_1, D_2)$ represents how many coordinates need to be corrupted (out of $n$ on average) for $D_1$ and $D_2$ to be indistinguishable. Then, given the original distribution $D$ and sufficiently large sample size, 
$\{D' : \dentry^1(D, D') \le \alpha\}$ represents the set of distributions that $\Athree^{\alpha}$ can show us after corruption, and thus $\dentry^{1}$ allows us to capture all possible actions of this adversary. Similarly, $\dentry^{\infty}$ captures all possible actions of $\Atwo^\rho$. We use $\dentry$ when both $\dentry^{1}$ and $\dentry^{\infty}$ apply. 

\begin{figure}
\centering
    \subfigure[Corruption of both coordinates]{\includegraphics[width=0.45\textwidth]{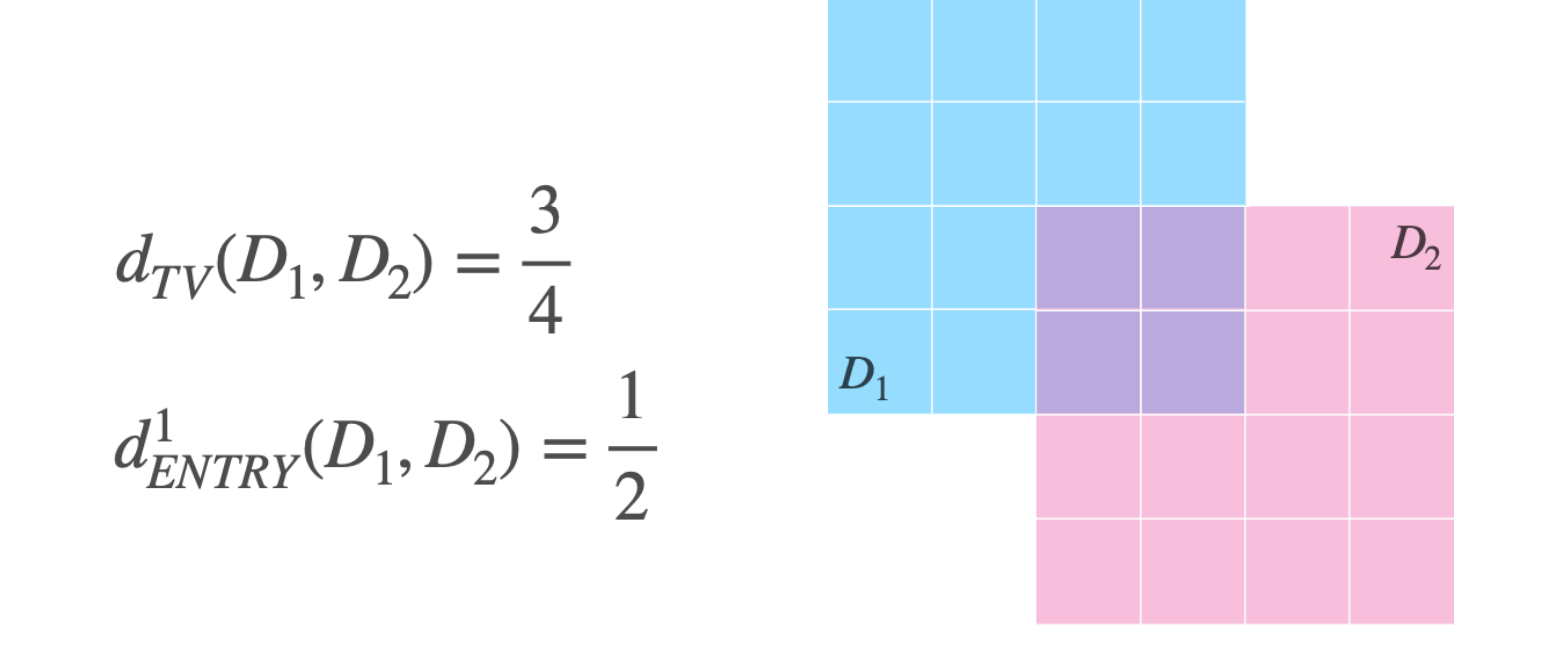}}
    \subfigure[Corruption of one coordinate]{\includegraphics[width=0.45\textwidth]{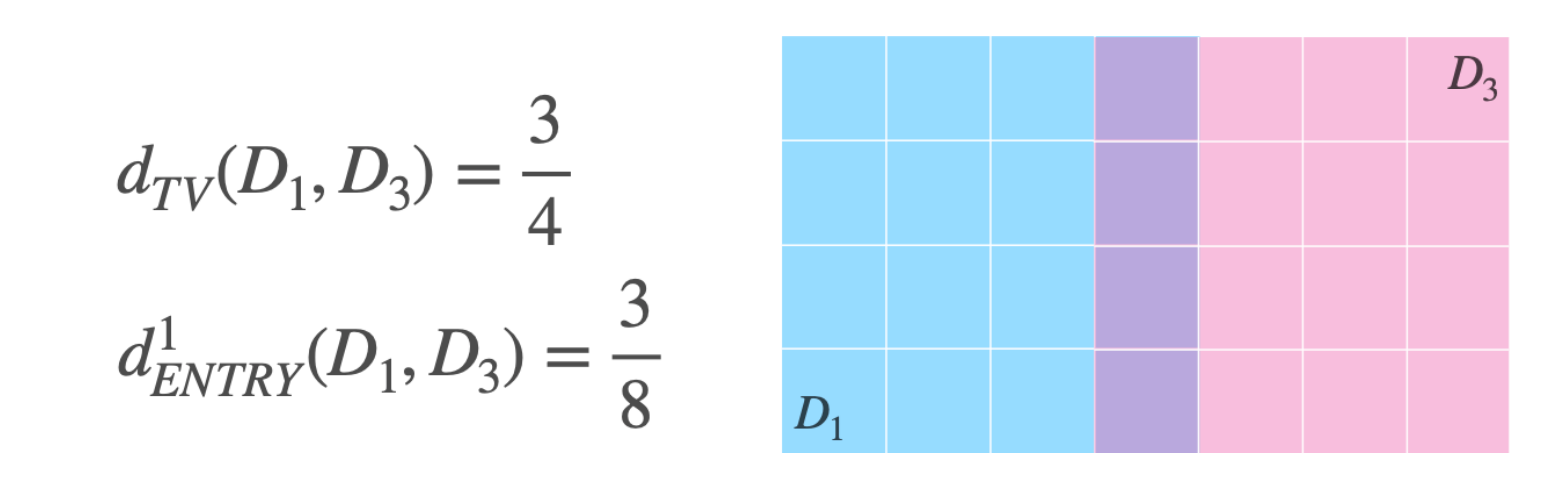}} 
\caption{\small Comparing $\dtv$ and $\dentry$: $D_1$ is the original 2D uniform distribution. $D_2$ is obtained after corruptions from $\Athree^{\alpha}$ with $\alpha = 1/2$, and $D_3$ after corruption from $\Athree^{\alpha}$ with $\alpha = 3/8$.}
\label{fig:disCompare}
\end{figure}

We compare $\dentry^1$ and $\dtv$ in Figure~\ref{fig:disCompare}. We consider a 2D uniform distribution $D_1$ and two corrupted versions $D_2$ and $D_3$. $D_2$ is obtained after an adversary corrupts both coordinates for samples from the upper-left quadrant of $D_1$ and one of the coordinates for samples in the lower-left and upper-right quadrant of $D_1$. $D_3$ is obtained after an adversary corrupts the horizontal coordinate for samples obtained from the left-most $3/4$ of $D_1$. In both cases, $\dtv(D_1,D_2) = \dtv(D_1, D_3) = 3/4$ since $3/4$ of the samples from $D_1$ are corrupted. $\dentry$ is different: From Definition~\ref{def:dentry}, we have $\dentry^1(D_1,D_2) = 1/2$ and $\dentry^1(D_1,D_3) = 3/8$, thus, we can distinguish the two.

\subsection{Information-theoretic Bounds for Gaussians}\label{sec:inf_analysis}
We analyze robust mean estimation under coordinate-level corruptions for Gaussian distributions, the de facto choice in the robust estimation literature. This choice enables us to draw comparisons to prior mean estimation approaches. Our results are summarized in Table~\ref{tab:summary}. We first show an impossibility result for arbitrary Gaussian distributions: in the general case, the information-theoretic analysis based on $\dtv$ and sample-level adversaries~\cite{tukey1975mathematics,diakonikolas2019robust} is tight even for coordinate-level corruption adversaries. However, we show that this result \emph{does not hold} for distributions that exhibit structure, i.e., redundancy across coordinates. We show that, for structured Gaussian distributions and corruptions that lead to a $\dentry$-bounded distribution shift, one must exploit the structure to achieve information-theoretically optimal error for mean estimation.

\begin{table}
\small
\center
\caption{Our results for robust mean estimation ($||\hat{\mu} - \mu||_{\Sigma}$) under $\Aone^\epsilon$ (sample-level adversary), $\Atwo^\rho$, and $\Athree^\alpha$ (coordinate-level adversaries). The results are for Gaussian distributions.}
\label{tab:summary}
\vskip 0.15in
\begin{threeparttable}
\begin{tabular}{lccc}
\toprule
Structure & $\Aone^\epsilon$ & $\Atwo^\rho$ & $\Athree^\alpha$ \\
\midrule
No Structure & $\Theta(\epsilon)$ & $\Omega(\rho \sqrt{n})$, $O(\rho n)$ & $\Theta(\alpha n)$ \\
Linear structure $A$& $\Theta(\epsilon)$ & $O(\rho \frac{n}{m_A})$ & $\Theta(\alpha \frac{n}{m_A})$\\
\bottomrule
\end{tabular}
\end{threeparttable}
\vskip -0.1in
\end{table}

\newparagraph{Mean Estimation of Arbitrary Gaussians}
We consider a Gaussian distribution $\mathcal{N}(\mu, \Sigma)$ with full rank covariance matrix $\Sigma$. We assume that observed samples are corrupted by a coordinate-level adversary. We first present a common upper-bound on the mean estimation error for both $\Atwo^\rho$ and $\Athree^\alpha$, and then introduce the corresponding lower-bounds.

We obtain an upper-bound on $\|\hat{\mu}- \mu\|_\Sigma$ by using Proposition~\ref{lowerBoundAdversaryProposition}: A sample-level adversary can simulate a coordinate-level adversary when $\epsilon = \alpha \cdot n$. But, for Adversary $\Aone^\epsilon$ the Tukey median achieves optimal error $\|\hat{\mu}- \mu\|_\Sigma = \Theta(\epsilon)$ when $\epsilon < 1/2$. Thus, the Tukey median yields error $O(\alpha n)$ for the coordinate-level adversaries $\Atwo^{\rho}$ (when $\rho=\alpha$) and $\Athree^\alpha$, when $\alpha \cdot n < 1/2$. Note that the condition $\alpha \cdot n < 1/2$ is necessary for achieving such an upper bound. When $\alpha \cdot n \geq 1/2$, the coordinate-level adversary is able to corrupt more than half of the samples in the worst case, leading to unbounded error, which shows exactly the power of the coordinate-level adversary.

We now focus on lower-bounds for the mean estimation error. We first consider adversary $\Atwo^\rho$ who can corrupt at most $\rho$-fraction of each coordinate in the samples. For this setting, the optimal estimation error depends on the \emph{disc} of the covariance matrix $\Sigma$, where disc is defined as:
\begin{definition}(disc)
For a positive semi-definite matrix $M$, define $s(M)_{ij} = M_{ij}/\sqrt{M_{ii} M_{jj}}$ and $\text{disc}(M) = \max_{x \in [-1,1]} \sqrt{ x^T s(M) x }$. 
\end{definition}
\begin{theorem}\label{thm:atwo_omega}
Let $\Sigma \in \R^{n\times n}$ be full rank.
Given a set of i.i.d. samples from $\mathcal{N}(\mu, \Sigma)$ where the set is corrupted by $\mathcal{A}_2^\rho$, any algorithm for estimating $\hat{\mu}$ must satisfy $\|\mu - \hat{\mu}\|_{\Sigma} = \Omega( \rho \cdot \text{disc}(\Sigma^{-1}) )$.
\end{theorem}
From this theorem, we obtain the next corollary for the mean estimation error for $\Atwo^\rho$:
\begin{corollary}\label{cor:atwo_omega}
Given a set of i.i.d. samples from $\mathcal{N}(\mu, \Sigma)$ where the set is corrupted by $\mathcal{A}_2^\rho$, any algorithm that outputs a mean estimator $\hat{\mu}$ must satisfy $\|\mu - \hat{\mu}\|_{\Sigma} = \Omega( \rho \sqrt{n} )$.
\end{corollary}
We see that there is a gap between the lower and upper bound on the mean estimation error for $\Atwo^\rho$. However, we show that such a gap does not hold for $\Athree^\alpha$. For $\Athree^\alpha$, the lower bound is the same as the upper bound presented above. Specifically, for the coordinate-fraction adversary $\Athree^\alpha$, it is impossible to achieve a mean estimation error better than $O(\alpha n)$ in the case of arbitrary Gaussian distributions:
\begin{theorem}
\label{noStructureAlpha}
Let $\Sigma \in \R^{n\times n}$ be full rank. 
Given a set of i.i.d. samples from $\mathcal{N}(\mu, \Sigma)$ where the set is corrupted by $\mathcal{A}_3^\alpha$, any algorithm that outputs a mean estimator $\hat{\mu}$ must satisfy $\|\hat{\mu} - \mu\|_\Sigma = \Omega(\alpha n)$. 
\end{theorem}
To gain some intuition, consider $\Athree^\alpha$ with $\alpha \geq \frac{1}{n}$. In this case, $\Athree^\alpha$ can concentrate all corruption in the first coordinate of all samples, and hence, we cannot estimate the mean for that coordinate. An immediate result is that for worst-case coordinate-corruptions, i.e., corruptions introduced by $\Athree^\alpha$, over arbitrary Gaussian distributions the mean estimation error is precisely $\|\mu - \hat{\mu}\|_{\Sigma} = \Theta(\alpha n)$.

\newparagraph{Mean Estimation of Structured Gaussians}
The previous analysis for $\Athree^\alpha$ shows that we cannot improve upon existing algorithms. However, real-world data often exhibit structural relationships between features such that one may be able to infer corrupted values via other visible values~\cite{wu2020att}. We show that in the presence of \emph{structure} due to dependencies, one must exploit the structure of the data to achieve information-theoretically optimal error for mean estimation. To show that structure is key, we focus on samples $x_i \in \R^n$ that lie in a low-dimensional subspace such that $x_i = Az_i$, where $A \in \R^{n\times r}$ represents the structure. \blue{Such low-rank subspace-structure is natural in many real-world scenarios and we assume linearity for the convenience of analysis. In fact, linear structure can also encode more complex structures (e.g., polynomials) if one considers an augmented set of features.} We assume $z_i$ comes from a non-degenerate Gaussian in $\R^r$. We consider a data sample $x = Az$ before corruption and assume that corruption is introduced in $x$. 

In this setting, the coordinate-level adversary has limited effect in mean estimation due to the redundancy that $A$ introduces. We can measure the strength of this redundancy with respect to coordinate-level corruption by considering its row space. 
The coordinates of $x = Az$, and hence, the corrupted data, exhibit high redundancy when many rows of $A$ span a small subspace.
We define a quantity $m_A$ to derive information-theoretic bounds for structured Gaussians.

\begin{definition}[$m_A$]
\label{def:ma}
Given matrix $A \in \R^{n\times r}$, $m_A$ is the minimum number of rows one needs to remove from $A$ to reduce the dimension of its row space by one.
\end{definition}

When $A=I$ is the identity matrix, it is $m_I = 1$ and we can remove any row to reduce its row space; we have low redundancy.
But, for $A = \left[\mathbf{e}_1, \dots, \mathbf{e}_1\right]^\top$ where $\mathbf{e}_1$ has $1$ in its first coordinate and $0$ in the others, $m_A = n$ since we need to remove all $\mathbf{e}_1$'s to reduce $A$'s row space.
It holds that $1 \leq m_A \leq n$.

We next show that the higher the value that $m_A$ takes, the weaker a coordinate-level adversary becomes due to the increased redundancy. 
Intuitively, the coordinate-level adversary has to spend more budget per sample to introduce corruptions that will counteract the redundancy introduced by matrix $A$. 
Theorem \ref{linear} shows that $\Atwo^\rho, \Athree^\alpha$ cannot alter the original distribution too far in $\dtv$, leading to information-theoretically tight bounds for mean estimation.

\begin{theorem}
\label{linear}
Given two probability distributions $D_1, D_2$ on $\R^n$ with support in the range of linear transformation $A$,
\begin{equation*}
(m_A/n)\cdot \dtv(D_1, D_2)  \le \dentry (D_1, D_2) \le \dtv(D_1, D_2)   
\end{equation*}
\end{theorem}

Here, $D$ with support in the range of $A$ means a distribution $D$ that is generated such that it lies on the subspace generated by $A$, i.e., there is zero measure outside of this subspace. Since $\dtv$ between two Gaussians is asymptotically equivalent to the Mahalanobis distance between them, we get the following corollary using $\dentry$.

\begin{corollary}\label{cor:upperBoundStructure}
Let $\mathcal{N}(\mu, \Sigma)$ be a Gaussian with support in the range of linear transformation $A$. For $\hat{\mu}$ such that $\dentry(\mathcal{N}(\mu, \Sigma),\mathcal{N}(\hat{\mu}, \Sigma)) \le \alpha$, $\|\mu - \hat{\mu}\|_\Sigma = O(\alpha \frac{n}{m_A})$.
\end{corollary}

 \cameraReady{Note that the upper bound above is under the condition that $\alpha < m_A/(2n)$ when the corruption is limited to missing entries, and $\alpha < m_A/(4n)$ when replacement is allowed. Otherwise, more than half of the samples can be corrupted and unrecoverable (the proof of Theorem~\ref{thm:lowerBoundStructure} provides the conditions for recovery, and the break points for mean estimation follow).}
 Corollary \ref{cor:upperBoundStructure} shows that $\Atwo^\rho$ (when $\rho=\alpha$) and $\Athree^\alpha$ can only shift structured distributions by $O(\alpha \frac{n}{m_A})$.  This result suggests that we can improve upon the previous $O(\alpha n)$ mean estimation guarantees. Furthermore, the following theorem proves that this upper bound is tight under $\Athree^\alpha$.

\begin{theorem}\label{thm:lowerBoundStructure}
Let $\mathcal{N}(\mu, \Sigma)$ be a Gaussian with support in the range of linear transformation $A$ and let $\Athree^\alpha$ adversarially corrupt the samples. Any algorithm that outputs a mean estimator $\hat{\mu}$ must satisfy $\|\mu - \hat{\mu}\|_\Sigma = \Omega(\alpha \frac{n}{m_A})$.
\end{theorem}

While our analysis focuses on Gaussian distributions, our analysis framework generalizes to any class of distributions that admits an efficient robust mean estimator under sample-level corruption, e.g. distributions with bounded covariance. This generality stems from our general reduction scheme between coordinate-level and sample-level corruption. Such extension is also feasible for results in Section~\ref{sec:algorithms}.

\section{Efficient Algorithms}\label{sec:algorithms}

We discuss efficient estimation algorithms the match the error bounds of our analysis in Section~\ref{sec:inf_analysis}. Building upon practical approaches in data cleaning~\cite{wu2020att,DBLP:journals/pvldb/RekatsinasCIR17} that leverage the structure in data to perform estimation over corrupted data, we study the practicality and theoretical guarantees obtained by such recover-and-estimate approaches for robust mean estimation.

\newparagraph{Two-Step Meta-Algorithm}
The meta-algorithm has two steps: 1) \emph{Recover}: use the dependencies across coordinates of the data (i.e., the structure) to recover the values of corrupted samples (when possible); 2) \emph{Estimate}: After fixing corruptions, perform statistical estimation using an existing mean estimation method (e.g., empirical mean estimation). For theoretical analysis, we require {\em exact recovery} in the first step, i.e., we seek to recover the true sample before corruption without any errors, and view the samples that cannot be exactly recovered as the remaining corruptions.

We first show that when the corruption is limited to missing entries, there exist efficient algorithms with polynomial complexity that perform exact recovery. We show that plugging these algorithms into the two-step meta-algorithm yields practical algorithms for robust mean estimation under coordinate-level corruption, and these algorithms match the information-theoretic bounds. We further study corruption due to replacements. Here, we show that exact recovery in the presence of coordinate-level corruptions is NP-hard by building connections to sparse recovery. To overcome such hardness, we propose a randomized algorithm with a probabilistic guarantee with respect to the recovery, and as such, mean estimation.  

\subsection{Efficient Algorithms for Missing Entries}\label{sec:algorithm_missing}
We show two computationally efficient algorithm instances that achieve near-optimal guarantees for known and unknown structure in the presence of missing entries. \cameraReady{Details of the recovery step in the two algorithms are provided in the supplementary material.} We assume a sufficiently large sample size (infinite in the case of unknown structure) for all the analysis.



\newparagraph{Mean Estimation with Known Structure} When matrix $A$ is known, we recover missing coordinates as follows: we solve the linear system of equations formed by the non-corrupted data in the sample and $A$ to estimate $z$, and then use this estimation to complete the missing values of $x$. Such a recovery step has a complexity of $O(r^3)$. Given the recovered samples, we proceed with mean estimation. 

The above algorithm achieves error $\Theta(\alpha \frac{n}{m_A})$: the best strategy of $\Atwo^\rho$ or $\Athree^\alpha$ is to corrupt coordinates so that recovery is impossible. To this end, a coordinate-level adversary must corrupt at least $m_A$ coordinates for a sample to make coordinate recovery impossible. 
The two-step approach of recovery by solving a linear system and mean estimation with the Tukey median over full samples is information theoretically-optimal. However, the Tukey median is computationally intractable, and we use the empirical mean to obtain a computationally efficient algorithm. This approach yields a near-optimal guarantee of $\tilde{O}(\frac{\alpha n}{m_A})$ that is tight up to logarithmic factors---here $\tilde{O}(\epsilon) = O(\epsilon \sqrt{\log(1/\epsilon)})$.

\begin{theorem}
\label{metaAlgThm}
Assume samples $x_i = Az_i$ and $z_i$ comes from a Gaussian such that $x_i\sim \mathcal{N}(\mu, \Sigma)$ with support in the range of linear transformation $A$. Given a set of i.i.d. samples corrupted by $\mathcal{A}_2^\rho$ (when $\rho=\alpha$) or $\Athree^\alpha$, recover missing coordinates by solving a linear system of equations then discard all unrecoverable samples. The empirical mean $\hat{\mu}$ of the remaining samples satisfies $\|\hat{\mu} - \mu\|_{\Sigma} = \tilde{O}(\frac{\alpha n}{m_A})$, while the Tukey median $\hat{\mu}_{\text{Tukey}}$ of the remaining samples satisfies $\|\hat{\mu}_{\text{Tukey}} - \mu\|_{\Sigma} = O(\frac{\alpha n}{m_A})$.
\end{theorem}

Theorem \ref{metaAlgThm} shows that for $A$ with $m_A \approx n$, while the strong adversary $\Athree^\alpha$ introduces corruptions that shift the observed distribution by $\dentry = \alpha$, it can only affect the mean estimation as much as the weaker adversary $\Aone^\epsilon$ (with $\epsilon = \alpha$), which shifts the observed distribution only by $\dtv = \epsilon$.
This result implies that recovery by leveraging the structure reduces the strength of $\Athree^\alpha$ (and also $\Atwo^\rho$ with $\rho = \epsilon$) to that of $\Aone^\epsilon$. In fact, $m_A= n-r+1$ for almost every $A$ with respect to the Lebesgue measure on $\R^{n\times r}$, so $m_A \approx n$ when $r$ is sufficiently low-dimensional. The above means that we can tolerate coordinate-level corruptions with large $\rho$ and $\alpha$ only if we first recover and then estimate.

\newparagraph{Mean Estimation with Unknown Structure} If $A$ is unknown, we can estimate it using the visible entries before we use them to impute the missing ones. 
We build on the next result: matrix completion can help robust mean estimation in the setting of $x_i = Az_i$ when $A$ is unknown but has full rank, in which case $m_A = n-r+1$. 
Corollary 1 by ~\citet{matrixCompletion} gives the conditions in which we can uniquely recover a low rank matrix with missing entries.
This result goes beyond random missing values and considers deterministic missing-value patterns. We state it as follows:
\begin{lemma}
\label{lemmaNowak}
Assume samples $x_i = Az_i$ and $z_i$ comes from a Gaussian such that $x_i\sim \mathcal{N}(\mu, \Sigma)$ with support in the range of $A$, but $A$ is unknown and full rank. 
If there exist $r+1$ disjoint groups of $n-r$ samples, and in each group, any $k$ samples have at least $r+k$ dimensions which are not completely hidden, all the samples with at least $r$ visible entries can be uniquely recovered.
\end{lemma}

The above leads to the next algorithm: We recover the missing coordinates via matrix completion. A typical algorithm for matrix completion is iterative hard-thresholded SVD (ITHSVD)~\cite{ithsvd}, which has a complexity of $O(TNnr)$, where $T$ is the maximum number of iterations and $N$ is the sample size. Then, we use either Tukey median or a empirical mean estimation. We next analyze the guarantees of this algorithm.

Matrix completion requires learning the $r$-dimensional subspace spanned by the samples. 
Samples with more than $r$ visible entries provide information to identify this subspace. Given the subspace, samples with at least $r$ visible entries can be uniquely recovered. For corruptions by $\Athree^\alpha$ we show:
\begin{lemma}
\label{lemma:metaUnknownA3}
If $A$ is unknown and the data is corrupted by $\Athree^\alpha$ where $\alpha \geq \frac{1}{n}$, we cannot recover any missing coordinate, otherwise we can recover all the samples with less than $m_A$ missing entries.  
\end{lemma}

If we combine this lemma with Theorem~\ref{metaAlgThm}, we can obtain optimal mean estimation error---we obtain the same error guarantees with Theorem~\ref{metaAlgThm}---using matrix completion only when the budget of $\Athree^\alpha$ is bounded as a function of the data dimensions, i.e., when $\alpha < 1/n$ (see A.11 in the supplementary material). These guarantees are information-theoretically optimal but pessimistic: $\Athree^\alpha$ can hide all coordinates from the same dimension which can be unrealistic. Thus, we focus on adversary $\Atwo^\rho$.





\begin{theorem}
\label{thm:metaUnknownA2}
Assume samples $x_i = Az_i$ and $z_i$ comes from a Gaussian such that $x_i\sim \mathcal{N}(\mu, \Sigma)$ with support in the range of $A$, but $A$ is unknown and full rank. Under $\Atwo^\rho$ where $\rho < \frac{m_A-1}{n+(m_A-1)(m_A-2)}$, the above two-step algorithm obtains $\|\mu - \hat{\mu}\|_{\Sigma} = \tilde{O}(\frac{\rho n}{m_A})$, while the Tukey median $\hat{\mu}_{\text{Tukey}}$ of the remaining samples satisfies $\|\hat{\mu}_{\text{Tukey}} - \mu\|_{\Sigma} = O(\frac{\rho n}{m_A})$.
\end{theorem}

\begin{lemma}\label{lm:recovery_rho}
Assume $A$ is unknown and full rank. Under $\Atwo^\rho$, if $\rho \geq \frac{m_A-1}{n}$, we cannot recover any corrupted sample; if $\rho < \frac{m_A-1}{n+(m_A-1)(m_A-2)}$, we can recover all samples with at least $r$ visible entries.  
\end{lemma}

\subsection{Discussion of Corruptions due to Replacements}
The above two-step approach relies on exact recovery to provide guarantees for mean estimation. As we showed in the previous section, exact recovery is possible in the presence of missing values. However, when corruptions are introduced due to adverarial replacements, recovery is NP-hard in general. We show this by reducing the problem of sparse recovery to it. Sparse recovery is the problem of finding sparse solutions to underdetermined systems of linear equations, and it is shown to be computationally intractable by~\citet{codingIntractability}. Details of the connection between the two can be found in the supplementary material.

Works on decoding and signal processing have proposed efficient algorithms for exact recovery under replacements with probabilistic guarantees. Typical algorithms include basis pursuit~\cite{decodingLP} and orthogonal matching pursuit~\cite{davenport2010analysis}. However, these algorithms pose strict conditions on $A$, and have guarantees only on a limited family of matrices (e.g. Gaussian matrices). We defer the details of those conditions and guarantees to the supplementary material. For a generic $A$, as the matrices we consider in our problem setting, existing algorithms fail to provide any guarantee.

To alleviate the difficulty of recovery for replacements, we propose a randomized algorithm that has probabilistic guarantees for exact recovery without posing strict restrictions on matrix $A$. We fit this algorithm into the aforementioned two-step meta-algorithm, and provide an analysis similar to the case of missing entries for known structure.

\begin{algorithm}
\caption{Recovery for Coordinate-level Replacements}
\label{alg:approx}
\SetAlgoLined
 \textbf{input:} $A\in \R^{n \times r}$, corrupted sample $\tilde{x} \in \R^n$, $c>0$\\
 \If{$\exists z$ such that $\tilde{x}=Az$}{\Return $\tilde{x}$}
 \For{$i = 1 \dots r^c$}{
  Uniformly at random, select $r$ out of $n$ rows of $A$ such that they are linearly independent\\
  $(\tilde{x}', A') \gets$ linear system of the corresponding $r$ coordinates\\
  Compute the solution $\hat{z}$ to $\tilde{x}' = A'z$ \\
  Store the $\hat{z}$ that has the smallest $\|\tilde{x}-A\hat{z}\|_0$ so far
  }
  \Return $A\hat{z}$
\end{algorithm}

For a corrupted sample $\tilde{x}$ that does not lie on the subspace generated by $A$, Algorithm~\ref{alg:approx} efficiently recovers a candidate solution $A\hat{z}$ that is not too far from $\tilde{x}$. More precisely, it returns $A\hat{z}$ such that $\|\tilde{x} - A\hat{z}\|_0 \le \frac{r}{\ln(r)}\|\tilde{x} - x\|_0$ with high probability, assuming the original sample $x$ is information-theoretically recoverable (less than $\frac{m_A}{2}$ coordinates are corrupted by the proof of Theorem~\ref{thm:lowerBoundStructure}). Algorithm \ref{alg:approx} requires solving a linear system of size $r$ multiple times, yielding a runtime of $O(nr^{4+c})$, where $c$ is a parameter chosen depending on how accurate we want the solution to be. Since Algorithm \ref{alg:approx} gives us a recovery routine for each corrupted sample, we have the following result on mean estimation.

\begin{theorem}
\label{thm:approx}
Preprocessing the data with Algorithm \ref{alg:approx} and then applying a robust mean estimator for Gaussians, e.g. of Diakonikolas et al. \yrcite{diakonikolas2019robust}, yields a mean estimate $\hat{\mu}$ such that $\|\hat{\mu} - \mu\|_2 = \tilde{O}(\frac{r}{\ln{r}}\cdot \frac{\alpha n }{m_A})$.
\end{theorem}

Our exploration of computationally efficient algorithms for mean estimation under coordinate-level replacements are preliminary. The randomized algorithm only works when the structure is known, and the guarantee it provides is probabilistic. Finding algorithms with better guarantees and recovery methods when the structure is unknown is an exciting direction for future research.

\cameraReady{\remark{Remark} We assume sufficiently large sample size and ignore the sampling error. In fact, the sample size and the probability when the error bound holds depend on the mean estimator in the \emph{Estimate} step of the meta algorithm. For example, in Theorem~\ref{thm:approx}, if we apply the robust mean estimator for Gaussians by~\citet{diakonikolas2019robust}, with sample size $N=\tilde{\Omega}(\frac{m_A^2 \log^5(1/\delta)}{\alpha^2})$, the upper bound holds with probability $1-\delta$.}

\section{Experiments}\label{sec:experiments}
We compare the two-step recover-and-estimate procedure against standard robust estimators under missing entries in real-world data.
We consider the following methods: 

1) \emph{Empirical Mean:} Take the mean for each coordinate, ignoring all missing entries. 
2) \emph{Data Sanitization:} Remove any samples with missing entries, and then take the mean of the rest of the data. 
3) \emph{Coordinate-wise Median (C-Median):} Take the median for each coordinate, ignoring all missing entries. 
4) \emph{Two-Step method with Matrix Completion (Two-Step-M):} Use iterative hard-thresholded SVD~\cite{ithsvd} to impute missing entries and compute the mean. We use randomized SVD~\cite{randomSVD} to accelerate. 

Here, we do not include the exact recovery method in Theorem~\ref{metaAlgThm} since the structure is unknown in real-world data. In the supplementary material, we provide synthetic experiments showing that the two-step method with exact recovery outperforms structure-agnostic methods by over 50\%.

\blue{In each experiment, we inject missing values by hiding the smallest $\epsilon$ fraction of each dimension, which introduces bias to the mean. It is the worst possible $\Atwo^\rho$ corruption for Gaussian distributions if the estimation is taking the empirical mean or median. Note that the worst $\Athree^\alpha$ corruption for empirical mean or median is hiding the tail of the coordinate with the highest variance. We do not include it in the experiments since such setting yields either zero error (when $\alpha < 1/n$, $n$ is the number of coordinates) or unbounded error (when $\alpha \geq 1/n$) if we apply Two-Step-M for data with linear structure (see Lemma~\ref{lemma:metaUnknownA3}).}

\begin{figure*}[t]
    \centering
    \subfigure[Leaf]{\includegraphics[width=0.33\textwidth]{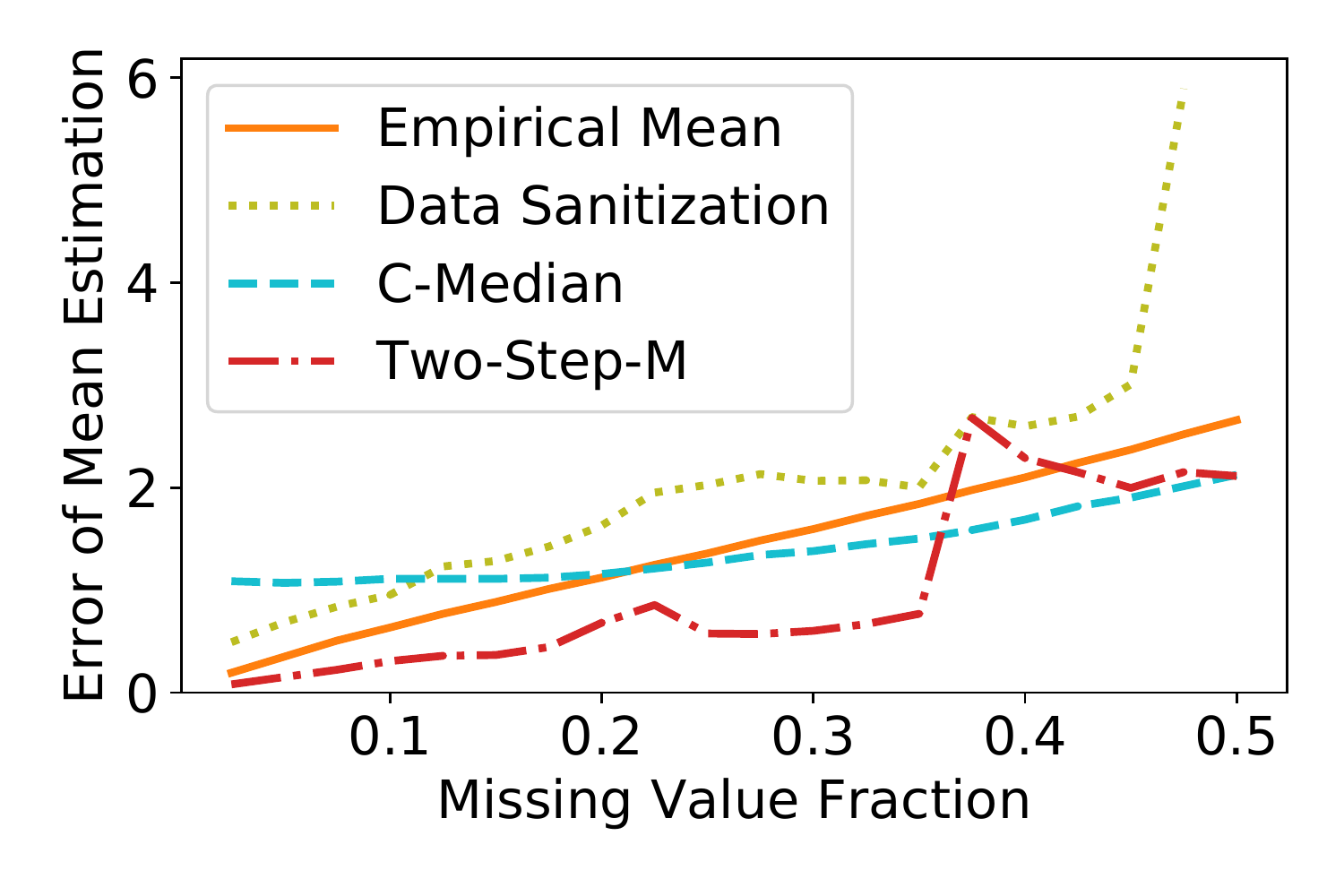}}
    \subfigure[Breast Cancer Wisconsin]{\includegraphics[width=0.33\textwidth]{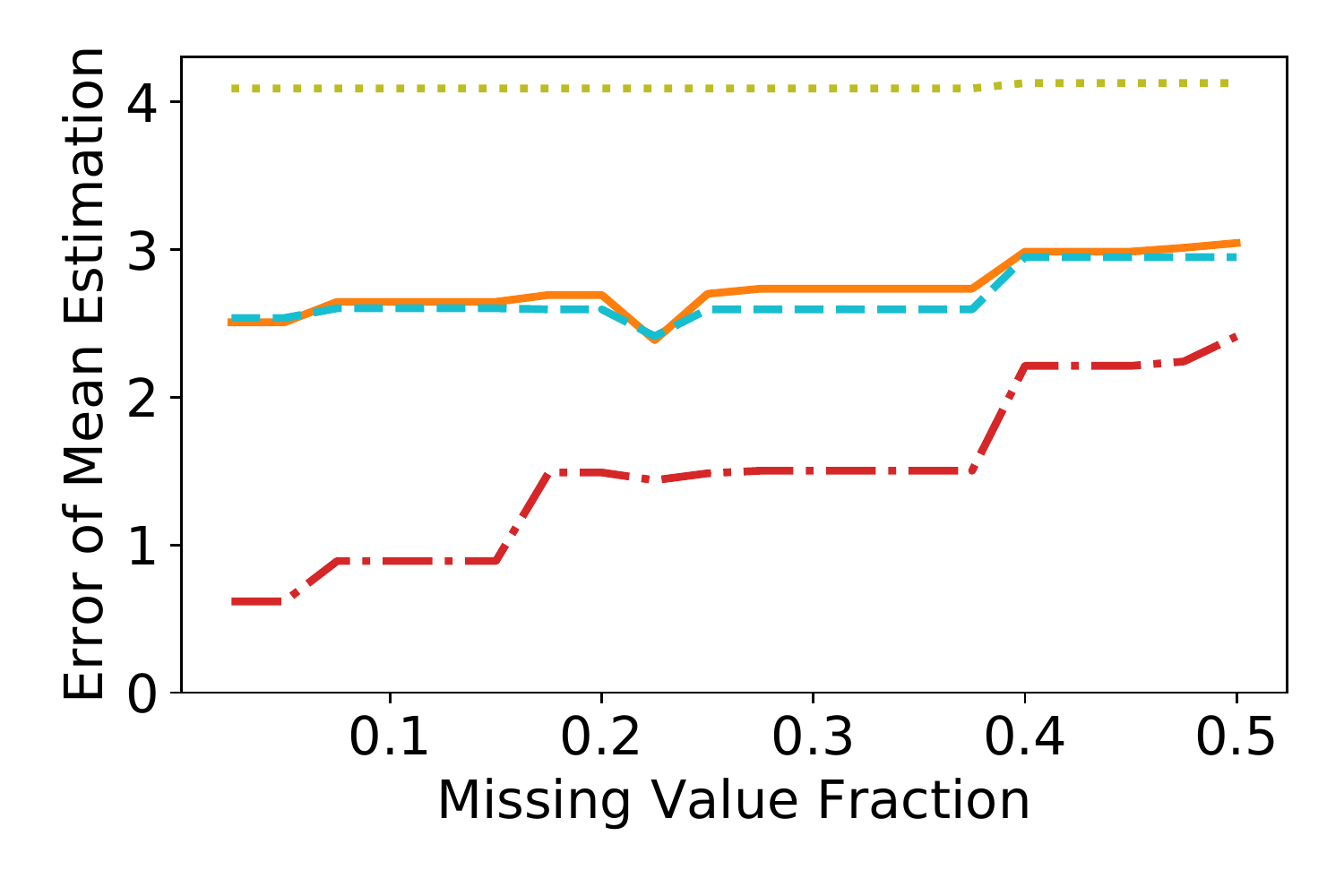}} 
    \subfigure[Blood Transfusion]{\includegraphics[width=0.33\textwidth]{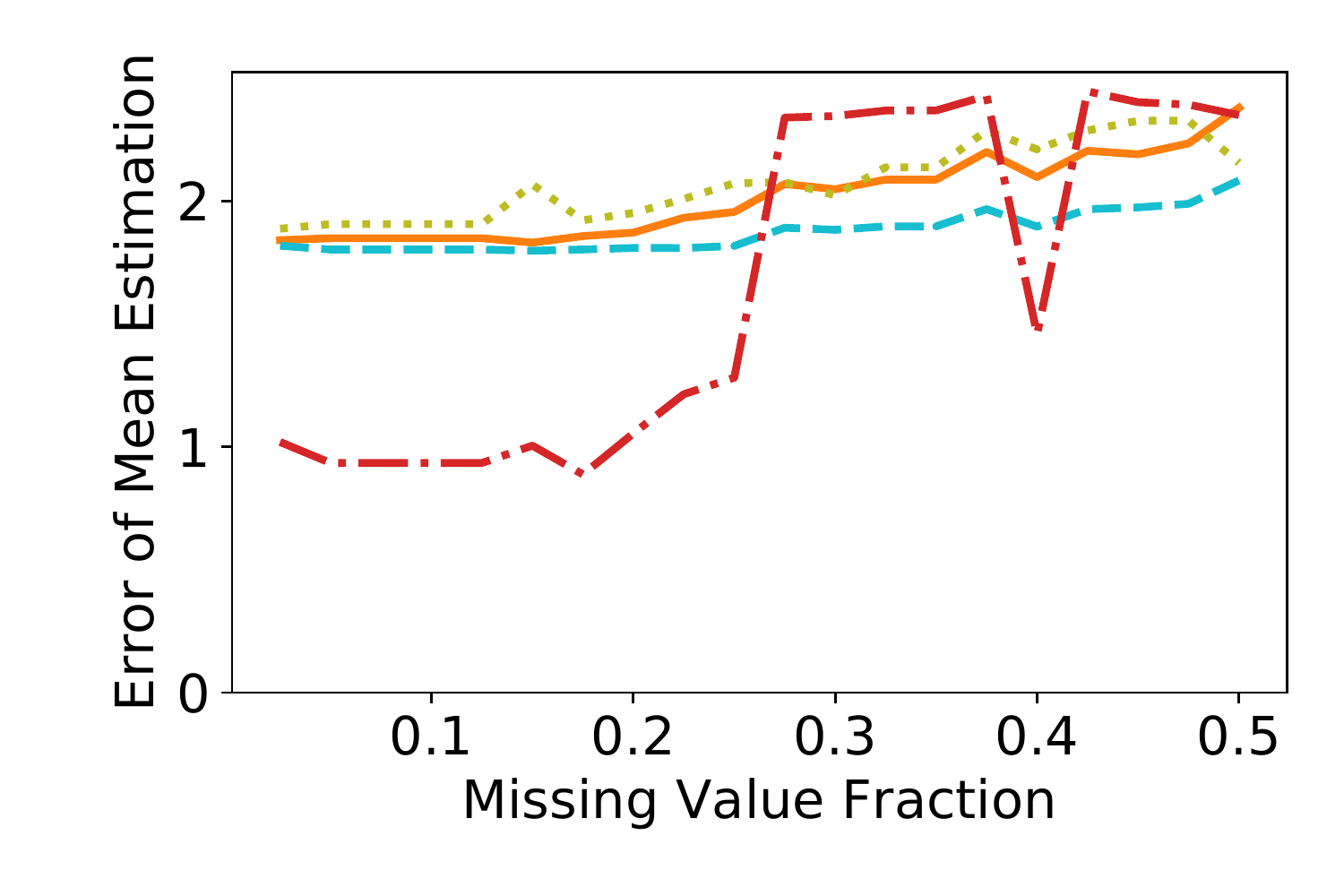}}
    \subfigure[Wearable Sensor]{\includegraphics[width=0.33\textwidth]{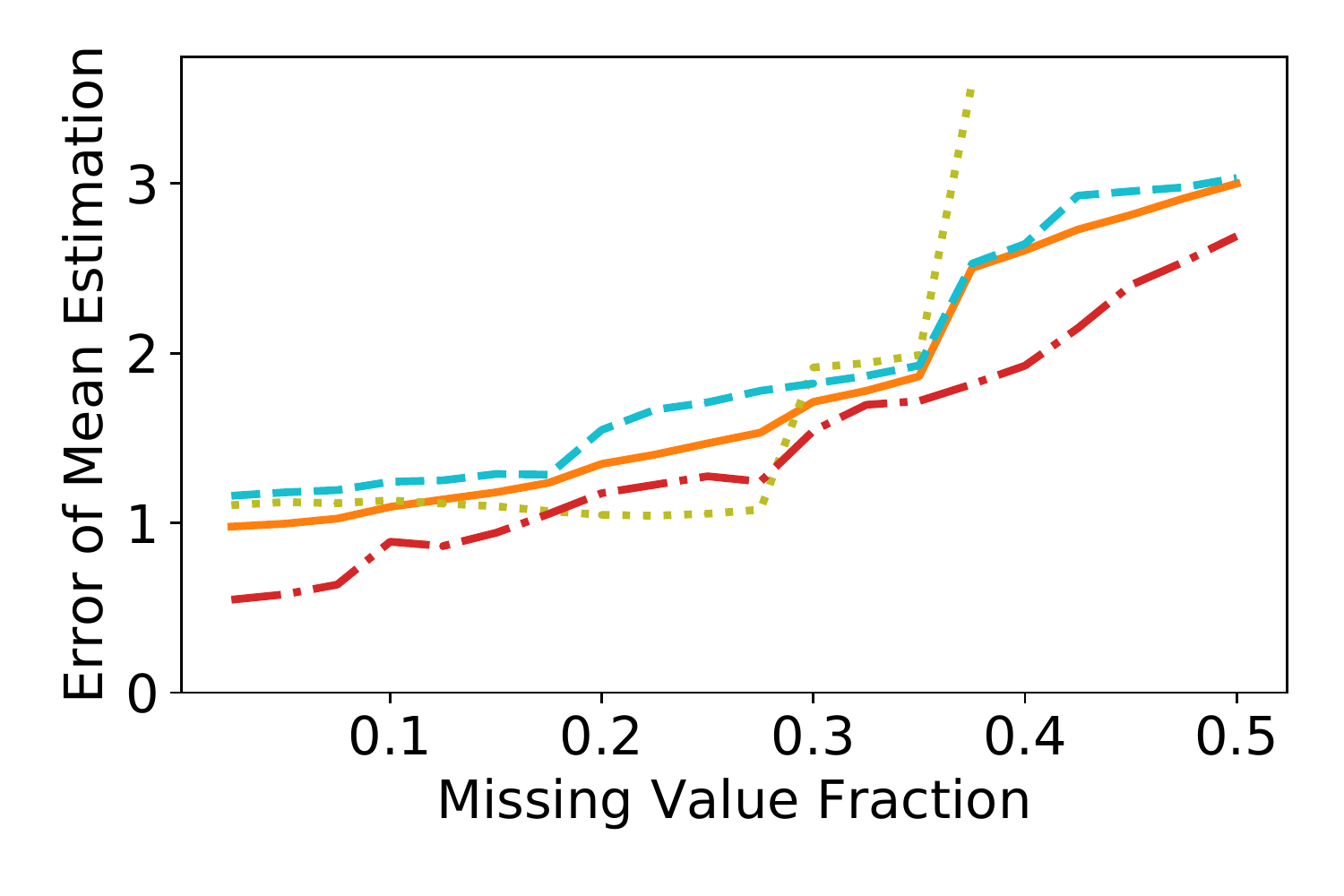}} 
    \subfigure[Mice Protein]{\includegraphics[width=0.33\textwidth]{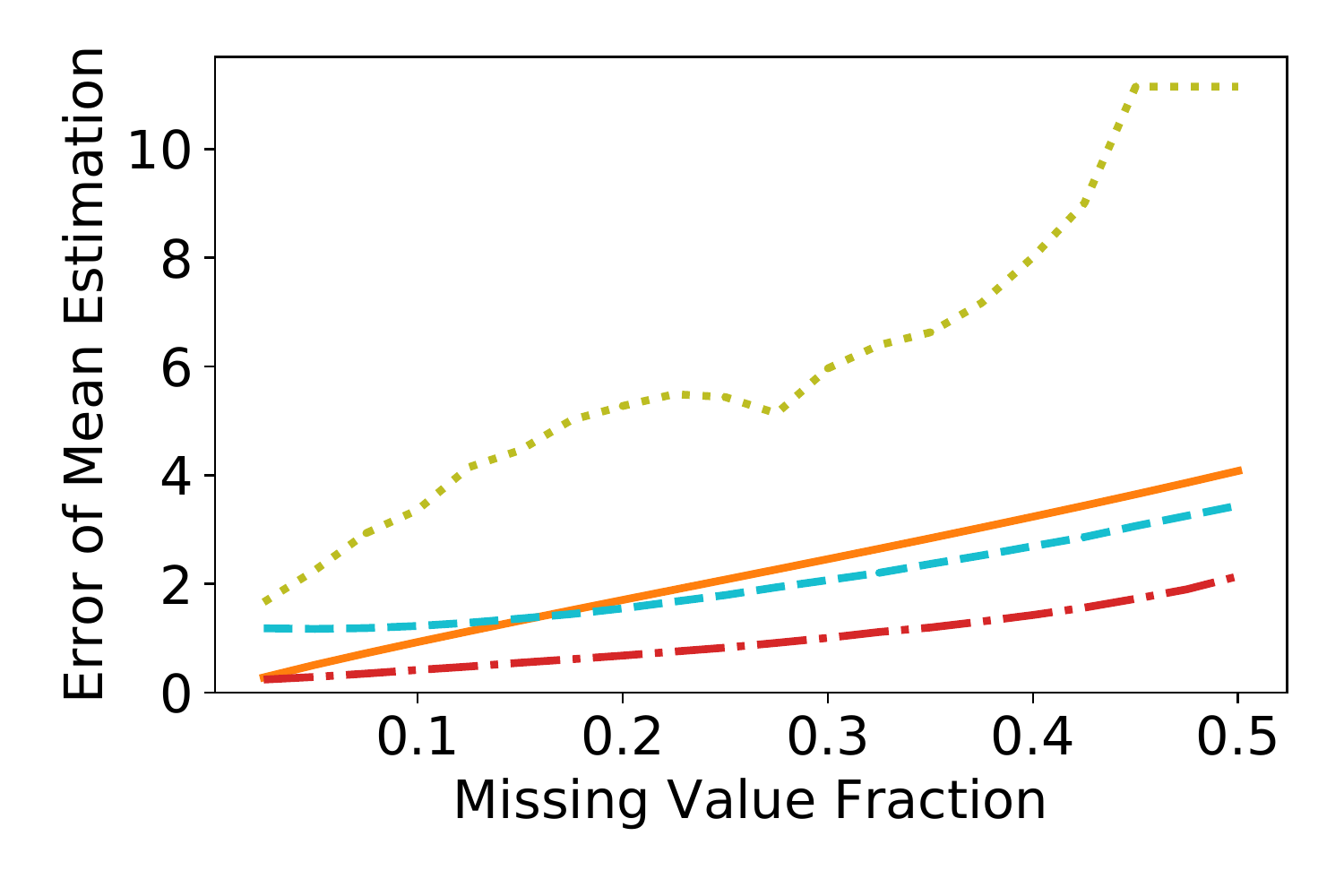}}
    \caption{Error of mean estimation on real-world data sets.}
    \label{fig:realworld}
    \vspace{-10pt}
\end{figure*}   

 We show that exploiting redundancy helps improve the robustness of mean estimation. We consider five data sets with unknown structure from the UCI repository~\cite{UCI}. Detailed information of these datasets is provided in the supplemental material.  For all the data sets, we report the $l_2$ error.  We use the empirical mean of the samples before corruption to approximate the true mean. We show the results in Figure \ref{fig:realworld}.
We find that Two-Step-M always outperforms Empirical Mean and C-Median on Breast Cancer Wisconsin, Wearable Sensor, Mice Protein Expression. 
For Leaf and Blood Transfusion, the error of Two-Step-M can be as much as $2\times$ lower than the error of the other methods for small budgets.
The estimation error becomes high only for large values of $\epsilon$.
Data Sanitization performs worse than Empirical Mean and C-Median.



Note that the error of Two-Step-M is not always increasing as the missing value fraction increases, even for synthetic data (see the supplemental material). This is because the performance of Two-Step-M depends on both the number of samples that can be recovered and the quality of the learned structure. When the missing value fraction is low, the conditions in Lemma~\ref{lemmaNowak} are satisfied, and the structure is learned exactly (for synthetic data) or well approximated (for the real data) via matrix completion. In this case, the error grows monotonically. When the missing value fraction is high, the conditions in Lemma~\ref{lemmaNowak} cannot be satisfied, and thus the quality of the learned structure is not guaranteed. In this case, the learned structure can be accurate or not by chance, and the error may not grow monotonically.

\section{Related Work}\label{sec:related}
Many works have studied problems in robust statistics. These include robust mean and covariance estimation~\cite{lai2016agnostic, daskalakis2018efficient, diakonikolas2019robust, cheng2019faster, kontonis2019efficient, zhu2019generalized}, robust optimization~\cite{charikar2017learning, diakonikolas2018sever, duchi2018learning, prasad2018robust}, robust regression~\cite{huber1973robust, klivans2018efficient, diakonikolas2019efficient, gao2020robust}, robust subspace learning~\cite{maunu2019robust, awasthi2020adversarially}, and computational hardness of robustness~\cite{hardt2013algorithms, klivans2014embedding, diakonikolas2019efficient, hopkins2019hard}.The works that are most relevant to ours focus on:


\newparagraph{Entry-level Corruption} 
~\citet{zhu2019generalized} define a family of non-parametric distributions for which robust estimation is well-behaved under a corruption model bounded by the Wasserstein metric, while we focus on the information-theoretic analysis of mean estimation.
\citet{wang2017robust} studies robust covariance estimation under a corruption model similar to $\Atwo^\rho$. They assume sparse covariance, which is different from our low-dimensional-subspace assumption.
Phocas~\cite{xie2018phocas} performs structure agnostic coordinate-wise estimation without considering recovery under $\Atwo^\rho$-type corruption. \citet{loh2018} study learning a sparse precision matrix of the data under cell-level noise but for a $\dtv$ adversary.

\newparagraph{Data Recovery} State-of-the-art methods for data imputation use data redundancy to obtain high accuracy even for systematic noise. SVD-based imputation methods~\cite{SVDspectral, dnamissing} assume linear relations across coordinates. There are other works that consider different models, including K-nearest neighbors, SVMs, decision trees, and even attention-based mechanisms~\cite{bertsimas2017predictive, wu2020att} to discover more complex non-linear structures. Our theoretical analysis provides intuition as to why these methods outperform solutions that only rely on coordinate-wise statistics.

\newparagraph{Robust Mean Estimators} \cameraReady{Previous works have studied the problem of robust mean estimation, and proposed computationally efficient estimators for high-dimensional data, such as truncated mean~\cite{burdett1996truncated}, geometric median~\cite{cohen2016geometric}, and the iterative filtering algorithm by~\cite{efficientRMEDiak}. Those works are orthogonal to ours, since they study robust mean estimation when there exhibit no structure in the data, while we focus on the effects of structure and recovery. Those estimators can be plugged into the two-step meta-algorithm, and the error bound analysis can be derived following our reduction scheme. }

\section{Conclusion}\label{sec:conclusion}
We studied the problem of robust mean estimation under coordinate-level corruption.
We proposed $\dentry$, a new measure of distribution shift for coordinate-level corruptions and introduced adversary models that capture more realistic corruptions than prior works.
We presented an information-theoretic analysis of robust mean estimation for these adversaries and showed that when the data exhibits redundancy one should first fix corrupted samples before estimation. Our analysis is tight. Finally, we study the existence of practical algorithms for mean estimation that matches the information-theoretic bounds from our analysis.



\section*{Acknowledgements}
The authors would like to thank Nils Palumbo for early discussions on this paper. Also, this work was supported by Amazon under an ARA Award, by NSF under grants IIS-1755676, IIS-1815538 and CCF-2008006.

\bibliography{struct_robust}
\bibliographystyle{icml2021}

\onecolumn

\section{Supplementary Material}\label{sec:proofs}
We now provide the proofs for all results stated in the main body of our work.
We also introduce the connection between recovery for replacements and sparse recovery.
Finally, we provide additional details and figures from our experimental evaluation that were omitted due to space constraints.

\subsection{Proof of Proposition~\ref{lowerBoundAdversaryProposition}}
\begin{proof}
    Suppose that $\Athree$ affects at least one entry in a subset $S$ of all samples. As at least one coordinate per sample is corrupted, $S$ must be at most an $\alpha$-fraction of all samples; since $\alpha \leq \epsilon / n$ the sample-level adversary can corrupt the entirety of every sample partially corrupted by the coordinate-level adversary, and thus, it is a stronger adversary given this condition. The proof for $\Atwo$ is similar.
\end{proof}

\subsection{Proof of Proposition~\ref{adversaryLemma}}
\begin{proof}
    If $\alpha,\, \rho \geq \epsilon$, similarly to the proof of Proposition \ref{lowerBoundAdversaryProposition}, $\Atwo$ and $\Athree$ can simulate $\Aone$ by placing all its corruptions on the $\epsilon N$ coordinates corrupted by $\Aone$. If $\alpha \ge \rho$, $\Athree$ can simulate $\Atwo$ by corrupting the coordinates corrupted by $\Atwo$ since $\Atwo$ can never corrupt more than $\rho$-fraction of coordinates in expectation. On the other hand, if $\alpha \le \rho / n$, $\Atwo$ can corrupt whatever coordinates $\Athree$ decides to corrupt since $\Athree$ cannot corrupt more than $\alpha n$-fraction of one coordinate. Thus, the three statements hold.
\end{proof}

\subsection{Proof of Theorem~\ref{thm:relation}}
\begin{proof}
\cameraReady{We first show that when $\alpha > 2 \alpha'$, $\Athree^\alpha$ has a way to make corruptions so that with probability at least $1-e^{-\Omega(\alpha^2 N)}$ it is indistinguishable whether the $N$ samples come from $D_1$ or $D_2$. From the definition of $\dentry^1$, if we take the coupling $\gamma$ that achieves the infimum, changing $\alpha'$ fraction of the entries per sample on average will make $D_1$ indistinguishable from $D_2$. Therefore, if the adversary corrupts the entries of the $N$ samples according to the coupling $\gamma$, by Hoeffding's inequality, the probability that more than $2 \alpha'$ fraction of the entries need to be changed to make it impossible to tell whether the samples come from $D_1$ or $D_2$ is less than $e^{-\Omega(\alpha^2 N)}$.}

\cameraReady{Then we show that when $\alpha < \alpha'/4$, no matter how $\Athree^{\alpha}$ makes corruptions, with probability at least $1-e^{-\Omega(\alpha^2 N)}$, we can tell that the $N$ samples come from $D_1$. Since $d_\text{ENTRY}^1(D_1, D_2) = \alpha'$, by Monge-Kantorovich duality theorem (see e.g. Theorem 5.10 of \cite{villani2009optimal}), there exists a function $u:(\mathbb{R} \cup \{\bot\} )^n \rightarrow [0,1]$, where $\bot$ denotes a missing entry, such that $u(x) - u(y) \le \frac 1n \|x-y\|_{0}$ and $\E_{D_1}[u(x)] - \E_{D_2}[u(x)] = \alpha'$.
This is because the optimal coupling of $D_1, D_2$ for $\dentry^1$ is represents the optimal to the primal Kantorovich problem where $c(x,y) = \frac{1}{n}\|I(x,y)\|_1$, while $u$ represents the optimal to the dual problem.
We use $u$ to distinguish whether the corrupted samples come from $D_1$ or $D_2$ by checking whether the expectation of $u$ according to the empirical distribution $\hat D$ that we observe is closer to the expectation corresponding to $D_1$ or $D_2$. By Hoeffding's inequality, the empirical distribution $\hat D_1$ of the $N$ samples before corruption satisfies $|\E_{D_1}[u(x)] - \E_{\hat D_1}[u(x)]| \le \alpha' / 4$ with probability at least $1-e^{-\Omega(\alpha^2 N)}$. After corruption, we have that $|\E_{\hat D}[u(x)] - \E_{\hat D_1}[u(x)]| \le \alpha$ by the bound on the number of corrupted entries and the Lipschitz property of $u$. Thus, with probability at least $1-e^{-\Omega(\alpha^2 N)}$, 
$|\E_{D_1}[u(x)] - \E_{\hat D}[u(x)]| \le \alpha' / 4 + \alpha < \alpha'/2$ while 
$|\E_{D_2}[u(x)] - \E_{\hat D}[u(x)]| > \alpha'/2$, which allows us to distinguish between $D_1$ and $D_2$.}

\cameraReady{In the case of the value-fraction adversary $\Atwo^\rho$ that can corrupt $\rho$-fraction of values in each coordinate, $\dentry^\infty$ can be bound similarly in $\Theta(\rho)$ by applying Hoeffding's inequality and Kantorovich duality theorem for each coordinate such that $u_i(x) - u_i(y) \le \|x_i - y_i\|_0$ and then comparing the mean for each coordinate. Therefore, for both $\Atwo^\rho$ and $\Athree^\alpha$, $\dentry$ is a tight characterization of the coordinate-level adversary.}
\end{proof}

\subsection{Proof of Theorem~\ref{thm:atwo_omega}}
\begin{proof}
Let $\mathrm{disc}(\Sigma^{-1}) = \max_{x \in [-1,1]} \sqrt{ x^T s(\Sigma^{-1}) x }$ and let $v$ be the vector with entries $(\Sigma^{-1}_{ii})^{-1/2}$.
To complete the proof, we will show that $\dentry^\infty(N(\mu, \Sigma), N(\mu + \rho v, \Sigma)) \le \rho$. 
To do this, we are going to use a hybrid argument showing that by only hiding $\rho$ fraction of the entries in the $i$-th coordinate,  $N(\mu, \Sigma)$ and $N(\mu + \rho \mathbf{e}_i / \Sigma^{-1/2}_{ii}, \Sigma))$ become indistinguishable where $\mathbf{e}_i$ is the vector that has $1$ in its $i^{\text{th}}$ coordinate and $0$ in the others. 
This is because, $\dtv(N(\mu, \Sigma), N(\mu + \rho \mathbf{e}_i / \Sigma^{-1/2}_{ii}, \Sigma))) \le \rho$. By applying this argument sequentially for every coordinate, $N(\mu, \Sigma)$ and $N(\mu + \rho v, \Sigma)$ are indistinguishable under an $\Atwo$ adversary. 
Since the total distance between $\mu$ and $\mu + \rho v$ in Mahalanobis distance is at least $\rho \cdot \mathrm{disc}(\Sigma^{-1})$, the theorem follows.
\end{proof}

\subsection{Proof of Corollary~\ref{cor:atwo_omega}}
\begin{proof}
We prove the following lemma that implies Corollary \ref{cor:atwo_omega} when combined with Theorem \ref{thm:atwo_omega}.

\begin{lemma}
\label{lm:psd_bounds}
For any $n \times n$ PSD matrix $M$, $\text{disc}(M) \in [\sqrt{n},n]$
\end{lemma}

We have that $s(M)$ is a PSD matrix with diagonal elements equal to 1.
Consider a random $x$ with uniformly random coordinates in $\{-1,1\}$.
Then, $\mathbb{E}[x^T s(M) x] = \text{Trace}( s(M) ) = n$. 
Thus, $\max_{x \in [-1,1]} \sqrt{ x^T s(M) x } \ge \sqrt{n}$. 
This lower bound is tight for $M=I$.

For the upper-bound, we notice that since $s(M)$ is PSD, it holds that $|s(M)_{ij} + s(M)_{ji}| \le 2$. 
To see this notice that $x^T s(M) x \ge 0$ for both $x = e_i + e_j$ and $x = e_i - e_j$.

Given this, we have that $x^T s(M) x \le \frac 12 \sum_{ij} |s(M)_{ij} + s(M)_{ji}| \le n^2$. 
This gives the required upper-bound. 
Notice that the upper-bound is tight for the matrix $M$ consisting entirely of $1$'s.
\end{proof}

\subsection{Proof of Theorem~\ref{noStructureAlpha}}
\begin{proof}
With a budget of $\alpha$, $\Athree$ can concentrate its corruption on one particular coordinate, say the first coordinate. 
If $\alpha n \ge 1$, we will lose all information for the first coordinate, making mean estimation impossible. 
Since $\alpha < 1/n$, $\Athree$ can corrupt $\alpha n$-fraction of first coordinates of all samples. 
Since the marginal distribution with respect to the first dimension is a univariate Gaussian, information-theoretically any mean estimator of the first coordinate must be $\Omega(\alpha n)$-far from the true mean of the first coordinate.

\end{proof}

\subsection{Proof of Theorem~\ref{linear}}
\begin{proof}
    First, we show the case for $\dentry^1$. 
    
    $\dentry^1(D_1, D_2) \leq \dtv(D_1,D_2)$ follows from 
    \begin{align*}
        \dentry^1(D_1, D_2) &= \inf_{\gamma \in \Gamma(D_1, D_2)} 
        \frac{\|\E_{(x,y)\sim \gamma}\left[I(x,y)\right]\|_1}{n}\\
        &= \inf_{\gamma \in \Gamma(D_1, D_2)} \E_{(x,y)\sim \gamma}\left[\frac{||x - y||_0}{n}\right]\\
        &\leq \inf_{\gamma \in \Gamma(D_1, D_2)} \Pr_{(x,y)\sim\gamma}\left[x\neq y\right]\\
        &=  \dtv(D_1,D_2)
    \end{align*}
    
    Then we show that $\dentry^1(D_1, D_2) \geq \frac{m_A}{n}\, \dtv(D_1,D_2)$.
    
    We first show that for any $x \not\in \ker(A)$, $||Ax||_0 \geq m_A$. Suppose by way of contradiction that $\Pi_i\, Ax$ is nonzero for fewer than $m_A$ values of $i$. Call the rows of $A$ $v_0^T,\dots,v_{n-1}^T$ and let $S$ be the subspace of $\R^r$ spanned by the $v_i$'s. As $x \not\in \ker(A)$, $Ax$ is nonzero. Hence, $\langle x,\,v_i\rangle$ is nonzero for some $i$ so $\Pi_S\,x$ is nonzero.

    Now, let $\mathcal{B}$ be a basis for $S$ containing $\Pi_S\,x$. Consider the subspace $S'$ of $S$ spanned by $\{v_i\mid \langle x,\,v_i\rangle = 0\}$. As $\Pi_{S'}\, x = 0$, $\Pi_S\,x$ cannot be an element of $S'$ and so  $\mathcal{B}$ is not a basis for $S'$. Thus, the dimension of $S'$ is less than that of $S$; as $|\{v_i\}|-|\{v_i\mid \langle x,\,v_i\rangle = 0\}| < m_A$ 
    we have a contradiction of the definition of $m_A$. Thus, if $x\neq 0 \in \R^r$, $\Pi_i\, Ax$ must be nonzero for at least $m_A$ values of $i$, and hence $||Ax||_0 \geq m_A$.
    
    Now, suppose that $(x,y)\sim\gamma$ for some $\gamma \in \Gamma(D_1,D_2)$. Then, $x = Ax'$ and $y = Ay'$ for some $x', y' \in \R^r$. If $x\neq y$, then $Ax' \neq Ay'$ so $x'-y' \not\in \ker(A)$. Thus 
    $$||A(x'-y')||_0 \geq m_A$$ 
    by the above, and so 
    $$\E_{(x,y)\sim \gamma}\left[||x - y||_0\right] \geq m_A\Pr_{(x,y)\sim\gamma}\left[x\neq y\right]$$
    
    Therefore, we have that
    \begin{align*}
        \dentry^1(D_1, D_2) &= \inf_{\gamma \in \Gamma(D_1, D_2)} 
        \frac{\|\E_{(x,y)\sim \gamma}\left[I(x,y)\right]\|_1}{n}\\
        &= \inf_{\gamma \in \Gamma(D_1, D_2)} \E_{(x,y)\sim \gamma}\left[\frac{||x - y||_0}{n}\right]\\
        &\geq \inf_{\gamma \in \Gamma(D_1, D_2)} \frac{m_A}{n} \Pr_{(x,y)\sim\gamma}\left[x\neq y\right]\\
        &= \frac{m_A}{n}\, \dtv(D_1,D_2)
    \end{align*}
    
    In the case of $\dentry^\infty$, the left hand side ($\dentry^\infty(D_1, D_2) \geq \frac{m_A}{n}\, \dtv(D_1,D_2)$) follows from above by using the fact that $\|x\|_1 \le n \|x\|_\infty$ for $x \in \R^n$. The right hand side follows from
    \begin{align*}
        \dentry^\infty (D_1, D_2) &=  \inf_{\gamma \in \Gamma(D_1, D_2)} \|\E_{(x,y)\sim \gamma}\left[I(x,y)\right] \|_\infty \\
        &= \inf_{\gamma \in \Gamma(D_1, D_2)} \max_{i} \Pr_{(x,y)\sim \gamma}\left[x_i \neq y_i\right] \\
        &\leq \inf_{\gamma \in \Gamma(D_1, D_2)} \Pr_{(x,y)\sim\gamma}\left[x\neq y\right]\\
        &=  \dtv(D_1,D_2)
    \end{align*}
    
    Therefore, the theorem holds for the $\dentry$ metric.
\end{proof}

\subsection{Proof of Corollary~\ref{cor:upperBoundStructure}}
\begin{proof}
We can obtain the given upper bound relating the distance to $\dtv$. Since $\dtv(\mathcal{N}(0,1), \mathcal{N}(\mu,1)) = erf(\frac{\mu}{2\sqrt{2}})$, for small $\mu > 0$, $erf(\frac{\mu}{2\sqrt{2}}) = \Theta(\mu)$. Then
\begin{align*}
\dtv(\mathcal{N}(\mu, \Sigma),\mathcal{N}(\mu', \Sigma)) &= \dtv(\mathcal{N}(0, I),\mathcal{N}(\Sigma^{-1/2}(\mu'-\mu), I))  \\
&= \dtv(\mathcal{N}(0, 1),\mathcal{N}(\|\Sigma^{-1/2}(\mu'-\mu)\|_2, 1)) \\
&= \dtv(\mathcal{N}(0, 1),\mathcal{N}(\|\mu'-\mu\|_\Sigma, 1)) = \Theta(\|\mu'-\mu\|_\Sigma)
\end{align*}
Applying Theorem \ref{linear}, we get that $\|\mu - \mu'\|_\Sigma = O(\alpha \frac{n}{m_A})$.
\end{proof}

\subsection{Proof of Theorem~\ref{thm:lowerBoundStructure}}
\begin{proof}
We prove the theorem for both missing values and replaced values. In the case of missing values, for the lower bound, $\Athree^\alpha$ may corrupt at most $\frac{\alpha n}{m_A}$-fraction of the samples so that the coordinates are non-recoverable and shift part of the original distribution to anywhere along the axes of missing coordinates. 
Then the proof similarly follows the lower bound proof for estimating the mean of a Gaussian corrupted by $\Aone^\epsilon$. 
Hence, since we cannot distinguish between two Gaussians that share $1-\frac{\alpha n}{m_A}$ of mass, $\|\hat{\mu} - \mu\|_{\Sigma} = \Omega(\alpha \frac{n}{m_A})$.

For $\Athree^\alpha$ that replaces values, we prove the following lemma and the theorem follows.
\begin{lemma}
\label{lemma:corruptionRecover}
The adversary corrupts $\delta$ coordinates of a sample. Let $\tilde{x}$ be the corrupted sample and $x^*=Az^*$ be the original. We can only information-theoretically recover $x^*$ from $\tilde{x}$ if and only if $\delta < \frac{m_A}{2}$. Furthermore, if $\delta < \frac{m_A}{2}$ then $\|\tilde{x}-Az^*\|_0 < \delta$ and $\|\tilde{x}-Az'\|_0 \ge m_A - \delta$ for any $z' \neq z^*$.
\end{lemma}

Assume that if $\delta < \frac{m_A}{2}$ then $\|\tilde{x}-Az^*\|_0 < \delta$ and $\|\tilde{x}-Az'\|_0 \ge m_A - \delta$ for any $z' \neq z^*$. This implies that we can consider all possible subsets $I \subseteq \mathcal{U}$ where $|I| = n-\frac{m_A}{2}$ and solve the linear system of equations of $\tilde{x}_I = A_I z$ and output the solution $z$, which achieves smallest hamming distance to $\tilde{x}$, as $z^*$. If $\delta \ge \frac{m_A}{2}$, it is information theoretically impossible to recover $x^*$ as $\arg \min_z \|\tilde{x}-Az\|_0$ may not be unique: since the corruptions are adversarial, $z^*$ may not be part of the set of minimizers.

Assume $\delta < \frac{m_A}{2}$. Without loss of generality, let $A$ be full rank. If not, the proof follows by replacing $r$ with $\mathrm{rank}(A)$ and considering the kernel of $A$. Let $A_i$ denote the $i$-th row of matrix $A$ and $\mathcal{U} = \{A_i: i \in [n]\}$. Let $\Delta$ denote the set of $A_i$'s that correspond to the corrupted coordinates of $\tilde{x}$ so that $|\Delta| = \delta$. Define  $\mathcal{S} \supseteq \Delta$ to be the smallest subset of $\mathcal{U}$ such that row space dimension (rank) of $A_{\mathcal{U}\setminus \mathcal{S}}$ is 1 less than that of $A$. By definition of $m_A$, $|\mathcal{S}| \ge m_A$. 

The entries corresponding to rows $\mathcal{U}\setminus \mathcal{S}$ are uncorrupted, so if we solve the linear system $A_{\mathcal{U}\setminus \mathcal{S}}z = \tilde{x}_{\mathcal{U}\setminus \mathcal{S}}$, we will a get a 1-dimensional solution space for $z$. Thus, any $z$ in this line will give at least $|\mathcal{U}\setminus \mathcal{S}|$ matching coordinates when multipied to $A$ with $x^*$. Now, we can generate $|\mathcal{S}|$ many solutions, each corresponding to the solution to the linear system $A_{\mathcal{U}\setminus \mathcal{S} \cup \{s\}}z = \tilde{x}_{\mathcal{U}\setminus \mathcal{S} \cup \{s\}}$ for each $s \in \mathcal{S}$. 

For $s$ that corresponds to an uncorrupted entry in $\tilde{x}$, the solution to the linear system is the true solution $z^*$ since none of the values in the system was corrupted. That gives us at least $|S|-\delta$ solutions out of $|S|$ solutions to be exactly $z^*$. Regardless of how the adversary corrupts the $\delta$ entries, if $\delta < \frac{m_A}{2}$, then the majority solution will always be $z^*$ since $|S|-\delta > \frac{m_A}{2} > \delta$.  Furthermore, for $z' \neq z^*$, $z'$ can match at most $|\mathcal{U}\setminus \mathcal{S}|+\delta \le n - m_A+\delta$ coordinates of $\tilde{x}$, i.e. $\|Az' - \tilde{x}\|_0 \ge m_A - \delta$. However, if $\delta \ge \frac{m_A}{2}$, then there is no clear majority so it is impossible to distinguish between the true solution and the other solution. In fact, when $\delta$ is strictly greater and corruptions adversarially chosen, $\|Az^* - \tilde{x}\|_0 = \delta$ and there exists some $z'$, $\|Az' - \tilde{x}\|_0 = m_A-\delta < \|Az^* - \tilde{x}\|_0$.

\end{proof}

\subsection{Details of the Recovery Steps in the Algorithms in Section~\ref{sec:algorithm_missing}}
\cameraReady{The input is the structure matrix $A$ and corrupted samples $\tilde{x}_i$ with $M_i$ missing entries for $i = 1,2,..,N$. When $A$ is known, we iterate over all samples. If $M_i \geq m_A$, we discard $\tilde{x}_i$. Otherwise, we remove the missing entries in $\tilde{x}_i$ and get $\tilde{x}'_i$, and also remove the rows in $A$ corresponding to those missing entries and get $A'$. Then we solve $A'z  = \tilde{x}'_i$ and recover sample $i$ with $Az$. When $A$ is unknown, we perform matrix completion on $\tilde{X}$, the $i^{\text{th}}$ row of which is $\tilde{x}'_i$. We first impute the missing entries with the coordinate-wise medians. Then we repeat the following procedure until convergence: 1) compute the rank-$m_A$ projection of $\tilde{X}$ 2) replace the entries that are missing initially with the corresponding entries from the projection. The details of the projection and the completion procedure can be found in~\cite{ithsvd}.}

\subsection{Proof of Theorem~\ref{metaAlgThm}}
\begin{proof}
Define $\epsilon$ be the fraction of samples that has at least one corrupted coordinate. Note that the coordinate-level adversary must corrupt at least $m_A$ coordinates of a sample to make his corruptions non-recoverable. 
Given that we can recover any sample with less than $m_A$ corrupted coordinates, we have that $\epsilon \le \frac{\alpha n}{m_A}$. 
If $D$ is the original distribution on $\R^n$ and $D'$ is the observed distribution, then $\dtv(D, D') \le \epsilon$. Since $\dtv$ between the two Gaussians is less than or equal to $\epsilon$, the Tukey median algorithm achieves $\|\hat{\mu}_{\text{Tukey}} - \mu\|_{\Sigma} = O(\frac{\alpha n}{m_A})$. On the other hand, removing $\epsilon$-fraction of a spherical Gaussian shifts the empirical mean by $O(\epsilon\sqrt{\log 1/\epsilon})$. Therefore, $\|\hat{\mu} - \mu\|_{\Sigma} = \tilde{O}(\frac{\alpha n}{m_A})$.
\end{proof}

\subsection{Proof of Lemma~\ref{lemma:metaUnknownA3}}
\begin{proof}
The adversary can simply hide one coordinate completely to prevent us from recovering that coordinate if at least one missing coordinate per sample in expectation is allowed. 
If the expected number of missing coordinates per sample is less than one, there must then be some positive fraction of samples with no missing coordinates; as we have infinite samples, we can select any $r+1$ disjoint sets of $n-r$ such samples to satisfy the conditions in Lemma \ref{lemmaNowak}.
\end{proof}

\subsection{Robust Mean Estimation Under Coordinate-fraction Adversaries with Unknown Structure}
\label{A3_unknown_Appendix}
If we consider the case where the data is corrupted by $\Athree^\alpha$, we have the following result.

\begin{theorem}
\label{thm:metaUnknownA3_Appendix}
Assume samples $x_i = Az_i$ and $z_i$ comes from a Gaussian such that $x_i\sim \mathcal{N}(\mu, \Sigma)$ with support in the range of $A$, but $A$ is unknown.  
Under corruption $\Athree^\alpha$ with budget $\alpha < \frac{1}{n}$, recover missing coordinates by solving the matrix completion problem and discard any unrecoverable samples. The empirical mean $\hat{\mu}$ of the remaining samples satisfies $\|\mu-\hat{\mu}\|_{\Sigma} = \tilde{O}(\frac{\alpha n}{m_A})$, while the Tukey median $\hat{\mu}_{\text{Tukey}}$ of the remaining samples satisfies $\|\hat{\mu}_{\text{Tukey}} - \mu\|_{\Sigma} = O(\frac{\alpha n}{m_A})$.
\end{theorem}

Theorem \ref{thm:metaUnknownA3_Appendix} is based on Theorem \ref{metaAlgThm} and the Lemma \ref{lemma:metaUnknownA3}.

\subsection{Proof of Lemma~\ref{lm:recovery_rho}}
\begin{proof}
First, We introduce the concept of hidden patterns. 
The set of coordinates missing from a sample forms its hidden pattern. 
We only consider the hidden patterns which have been applied to infinitely many samples as if only finitely many samples share a pattern, the adversary could hide those samples completely with $0$ budget. 

When $\rho \geq \frac{m_A-1}{n}$, the adversary is able to hide $m_A-1$ entries for every sample, and we cannot learn the structure from samples with only $r$ visible entries.

When $\rho < \frac{m_A-1}{n}$, the adversary does not have enough budget to hide $m_A-1$ entries of all the samples, so there exist some patterns with at least $r+1$ coordinates visible. 
We use $M$ to denote the number of such patterns, and $p_l$ to denote the probability of the $l^{th}$ pattern $P_l$, $l = 1,2,\dots, M$.

Next, we show a necessary condition for the adversary to prevent us from learning the structure. 
Since we have infinitely many samples, one group of $n-r$ samples satisfying the conditions in Lemma \ref{lemmaNowak} is enough to learn the structure since we can find another $r$ groups by choosing samples with the same hidden patterns.

It is obvious that the adversary has to hide at least one coordinate per pattern, otherwise we have infinitely many samples without corruption. 
No matter what the patterns the adversary provides, we try to get the samples satisfies the conditions in Lemma \ref{lemmaNowak} by the following sampling procedure.

\begin{enumerate}
\item Start with one of the patterns, pick $r+1$ visible coordinates of it to form the initial visible set $V_1$. Take one sample from this pattern to form the initial sample group $G_0$. Mark this pattern as checked.
\item For $K = 1,2,3,\dots, M-1$, take one of the unchecked patterns and check if it contains at least one visible coordinate not in $V_K$, the current visible set. If so, take one sample $x_K$ from it and pick any one of its visible coordinates $v_K \notin V_K$. Add $x_K$ to the sample group and $v_K$ to the visible set: $G_{K+1} = G_{K} \cup \{ x_K \}$, $V_{K+1} = V_{K} \cup \{ v_K \}$. If not, skip it. Mark the pattern as checked.
\end{enumerate}

We show by induction that any $k (k \leq K)$ different samples in $G_K$ have at least $r+k$ coordinates in $V_K$ not completely hidden. It is trivial that the property holds for $K = 1$. Assume that the property holds for $K$. According to the sampling procedure, when a new sample $x_K$ comes, it has at least one visible coordinate $v_K$ not in $V_K$. Consider any $k (k \leq K+1)$ different samples in $G_{K+1}$. If the $k$ samples don't include the new sample $x_K$, by the induction assumption they have at least $r+k$ coordinates in $V_K \subset V_{K+1}$ not completely hidden. If $x_K$ is one of the $k$ samples, again by the induction assumption the other $k-1$ samples have at least $r+k-1$ coordinates in $V_K$ not completely hidden, plus $v_K$ of $x_K$ is also not hidden, so there are at least $r+k$ coordinates in $V_{K+1}$ not completely hidden. Thus, the property also holds for $K+1$. By induction, any $k$ distinct samples from the group we get at the end of step 2 have at least $r+k$ coordinates not completely hidden, which means if the group has at least $n-r$ samples, the conditions in Lemma \ref{lemmaNowak} can be satisfied.

Denote the set of the patterns being picked as $\mathcal{P}_P$ and the set of the patterns being skipped as $\mathcal{P}_S$. Based on the previous analysis, the adversary has to manipulate the patterns so that $|\mathcal{P}_P| \leq n-r-1$, in which case the visible set cannot cover all the coordinates, which means there exists at least one common hidden coordinate for the patterns in $\mathcal{P}_S$ (otherwise the pattern where that coordinate is visible should have been picked). Since the fraction of hidden entries in that common coordinate is less than or equal to $\rho$, the sum of the probabilities of the patterns in $\mathcal{P}_S$ satisfies $\sum_{l: P_l \in \mathcal{P}_S} p_l \leq \rho$. Since all the patterns have at least one missing coordinate, we also have $p_{l:P_l \in \mathcal{P}_P} \leq \rho$. Thus, we have $\sum_{l=1}^{M}p_l \leq (n-r) \rho$. In such a case, the overall fraction of missing entries $\eta$ satisfies
\begin{align*}
\eta &\geq \sum_{l=1}^{M}p_l \frac{1}{n} + (1-\sum_{l=1}^{M}p_l) \frac{n-r}{n} \\
     &\geq (n-r) \rho \frac{1}{n} + (1-(n-r)\rho) \frac{n-r}{n}
\end{align*}

The first inequality holds because for the samples with at least $r+1$ visible entries, there are at least $1$ missing entries per sample, and for the samples with less than $r+1$ visible entries, there are at least $n-r$ missing entries per sample. In addition, $\eta$ also satisfies $\eta \leq \rho$, so we have $\rho \geq \frac{n-r}{n+(n-r-1)(n-r)} = \frac{m_A-1}{n+(m_A-1)(m_A-2)}$, which is a necessary condition for the adversary. Thus, if $\rho < \frac{m_A-1}{n+(m_A-1)(m_A-2)}$, we can learn the structure and impute all the samples with at least $r$ visible entries.
\end{proof}

\subsection{Proof of Theorem~\ref{thm:approx}}
\begin{proof}
Algorithm \ref{alg:approx} and the following analysis borrows significantly from the randomized approximation algorithm for the NP-hard problem Min-Unsatisfy in the work of  \cite{berman2001approximating}. 

If $\tilde{x}$ is in the subspace generated by $A$ (i.e. there exists $z$ such that $\tilde{x} = Az$), then the point $\tilde{x}$ must be either an uncorrupted point or a corrupt point where at least $m_A$ coordinates corrupted. This is because $m_A$ is equal to the minimum number of coordinates needed to move a point on the subspace generated by $A$ to another point on the subspace. 

If the point does not lie on the subspace, it is clearly a corrupted point. If this point is corrupted by more than $m_A/2$, it would be information-theoretically impossible to recover the true point as is shown in Lemma~\ref{lemma:corruptionRecover}. However, changing this point to a different point on the subspace would not matter in the reduction to a robust mean estimation algorithm since it is a outlier either way.

Now assume this point is corrupted by at most $m_A/2$. By Lemma 2 and Theorem 1 of \cite{berman2001approximating}, the for loop results in the best $\tilde{z}$ such that $\|\tilde{x} - A\tilde{z}\|_0 \le \frac{r}{c\ln r} \|\tilde{x} - x^*\|_0$ with high probability. In the case that the point is only corrupted by at most $\frac{c m_A \ln r}{2r}$ coordinates, then $\|\tilde{x} - A\tilde{z}\|_0 \le \frac{m_A}{2}$. But the only point on the subspace such that it only needs at most $\frac{m_A}{2}$ coordinate changes to $\tilde{x}$ is $x^*$. Then this approximation algorithm performs exact recovery of $x^*$ when at most $O(\frac{m_A \ln r}{r})$ coordinates are corrupted. Therefore, preprocessing points with this algorithm and then applying robust mean estimation yields a mean estimate $\hat{\mu}$ such that $\|\hat{\mu} - \mu\|_2 = \tilde{O}(\frac{r}{\ln{r}}\cdot \frac{\alpha n }{m_A})$.

\end{proof}

\subsection{Connection between Recovery for Replacements and Sparse Recovery}\label{sec:reduction}
We first show that the reduction from exact recovery for replacements to sparse recovery.
Let $x^* = Az^*$ be the uncorrupted sample, and $\tilde{x}$ be the same sample with no more than $\delta$ coordinates corrupted.
Let $e^* = \tilde{x} - x^*$, where $\|e^*\|_0 \leq \delta$, and we have $\tilde{x} = Az^* + e^*$. Take a non-trivial matrix $F \in \mathbb{R}^{p \times n}$ ($p < n$) which satisfies $FA = 0$. Apply $F$ to $\tilde{x}$ we have $y = F(A\tilde{x} + e^*) = Fe^*$. 
$x^*$ can be recovered if we know what $e^*$ is, so the recovery of $x^*$ can be reduced to recovering $e^*$ from $y$.

\citet{decodingLP} show that when $\delta < \frac{m_A}{2}$, we can get the exact $e^*$ by solving
\begin{equation*}
\label{eq:l0}
  \min_{e \in \mathbb{R}^n} \|e\|_0 
\text{ subject to } Fe = y 
\end{equation*}

Therefore, we reduce the problem of exact recovery for replacements to the problem of sparse recovery.

On the other hand, sparse recovery can also be reduced to exact recovery for replacements. Given the sparse recovery problem shown above, we take $A \in \mathbb{R}^{n \times r}$ such that the columns of $A$ span the null space of $F$, and $\tilde{x} \in \mathbb{R}^n$ such that $F\tilde{x} = y$. For any $e$ satisfying $Fe = y$, we have $F\tilde{x} = Fe$, and therefore $\tilde{x} = e + Az$ for some $z \in \mathbb{R}^r$. In addition, for any $z \in \mathbb{R}^r$, we can find $e = \tilde{x}-Az$ satisfying $Fe=y$. Thus, solving the exact recovery problem $\min_{z \in \mathbb{R}^r} \| \tilde{x} - Az \|_0$ will also give the sparsest $e$.

\subsection{Computationally Efficient Algorithms for Sparse Recovery}

Basis pursuit (BP)~\cite{decodingLP} and orthogonal matching pursuit (OMP)~\cite{davenport2010analysis} are computationally efficient algorithms that get the exact $e^*$ defined in Section~\ref{sec:reduction} under certain conditions.

BP approximates the sparse recovery problem by 
\begin{equation}
\min_{e \in \mathbb{R}^n} \|e\|_1 
\text{ subject to } Fe = y
\end{equation}

which is convex and can be solved by linear programming.

OMP is a greedy algorithm that gives a solution by the following procedure:
\begin{itemize}
    \item \textbf{Step 1} Initialize the residual $r^0 = y$, the index set $\Lambda^0 = \emptyset$, and the iteration counter $l=0$.
    \item \textbf{Step 2} Find the column of $F$ that has the largest inner product with the current residual and add its index to the index set: $\Lambda^{l+1} = \Lambda^{l} \cup \{ \argmax_{i} |F_i^T r^l| \}$.
    \item \textbf{Step 3} Update the estimation and the residual: $e^{l+1} = \argmin_{v:\supp{(v)} \subseteq \Lambda^{l+1} } \| y-Fv \|_2$, $r^{l+1} = y - Fe^{l+1}$.
    \item \textbf{Step 4} Output $e^{l+1}$ if converged, otherwise increment $l$ and return to Step 2.
\end{itemize}

\newparagraph{RIP-based Guarantee for BP and OMP}
\citet{decodingLP} introduce the Restricted Isometry Property (RIP) that characterizes the orthonormality of matrices when operating on sparse vectors. 

\begin{definition}{(RIP)}
A matrix $F$ satisfies the RIP of order $k$ if
\begin{equation}
    (1 - \zeta)\|c\|^2 \leq \|F c\|^2 \leq (1 + \zeta)\|c\|^2
\end{equation}
for all real coefficients $c$ with $\| c \|_0 \leq k$. $\zeta \in (0,1)$ is a constant.

\end{definition}
Based on the RIP condition, \citet{cai2013sharp} and \citet{davenport2010analysis} show that BP and OMP recover a $K$-sparse $e$ exactly when $F$ satisfies certain RIP conditions, and we state their results here as the following lemma.

\begin{lemma}
\label{lemma:BP_OMP_RIP}
 Consider the problem of recovering a $K$-sparse $e^*$ from $Fe^*$. BP recovers $e^*$ exactly if $F$ satisfies the RIP of order $K$ with $\zeta < \frac{1}{3}$; OMP recovers $e^*$ exactly in $K$ iterations if $F$ satisfies the RIP of order $K+1$ with $\zeta < \frac{1}{3\sqrt{K}}$.
\end{lemma}

\newparagraph{Matrices Satisfying RIP}
Gaussian matrices $F \in \mathbb{R}^{p \times n}$ whose entries are independent realizations of $\mathcal{N}(0, \frac{1}{p})$ satisfy RIP with certain probability, which can be shown by the following lemma that comes from \citet{baraniuk2008simple}. 

\begin{lemma}
\label{lemma:gaussianRIP}
Consider a Gaussian matrix $F \in \mathbb{R}^{p \times n}$ whose entries are independent realizations of $\mathcal{N}(0, \frac{1}{p})$. For a given $0<\zeta<1$, there exist constants $c_1, c_2 > 0$ such that the RIP of order $k$ ($k \leq c_1 p / log(n/k)$) with $\zeta$ holds for Gaussian metrics $F$ with probability at least $1-2e^{-c_2 p}$. $c_1, c_2$ have the relation that $c_2 = c_0(\zeta/2) -c_1[1+(1+\log (12/\zeta)) / \log(n/k)]$, where $c_0(x) = x^2/4-x^3/6$ is a function.
\end{lemma}

Note that one can choose sufficiently small $c_1$ so that $c_2 > 0$.

\newparagraph{Recovery for Gaussian $A$}
When the structure $A$ is a Gaussian matrix, we can perform efficient recovery by BP or OMP, with the following guarantee:
\begin{theorem}
Suppose $x^* = Az^*$ is an uncorrupted sample, where $A \in \mathbb{R}^{n \times r}$ is a Gaussian matrix whose entries are i.i.d. realizations of some Gaussian random variable $\mathcal{N}(0, \sigma^2)$. $\tilde{x} = x^*+e^*$ is the corrupted version of $x^*$, and $e^*$ is $K$-sparse with $K < m_A/2$. $x^*$ can be recovered exactly from $\tilde{x}$ with probability at least $1-2e^{-c_2 p}$ by the following procedure:
\begin{itemize}
    \item \textbf{Step 1:} Let $k = K$. Choose $\zeta < \frac{1}{3}$, and $c_1$ so that $c_2 = c_0(\zeta/2) -c_1[1+(1+\log (12/\zeta)) / \log(n/k)] > 0$. Choose $p$ such that $k \leq c_1 p / log(n/k)$.
    \item \textbf{Step 2:} For $i$ from $1$ to $p$, randomly sample a vector $v_i$ from the unit sphere in the null space of $A^T$, and a scalar $u_i$ from the chi-square distribution $\chi^2_n$ of degree $n$. Construct $F \in \mathbb{R}^{p \times n}$ where the $i^{\text{th}}$ row is $\sqrt{\frac{1}{p} u_i} v_i^T$.
    \item \textbf{Step 3:} Solve $Fe = F\tilde{x}$ by BP and denote the result as $\hat{e}$. Output $\hat{x} = \tilde{x}-\hat{e}$.
\end{itemize}
\end{theorem}

\begin{proof}
Since each row of $F$ is sampled from the null space of $A^T$, we have $FA = 0$ and $F\tilde{x} = Fe$. Because $A$ is a Gaussian matrix, $v_i$ is a random unit vector sampled from the $n$ dimension sphere, and each entry of it can be considered as i.i.d. samples from $\mathcal{N}(0, \frac{1}{p})$ after we multiply it by $\sqrt{\frac{1}{p} u_i}$. Hence, $F$ is a Gaussian matrix whose entries are independent realizations of $\mathcal{N}(0, \frac{1}{p})$. The satisfaction of the RIP condition and the exact recovery property follow from Lemma~\ref{lemma:gaussianRIP} and ~\ref{lemma:BP_OMP_RIP} respectively. Therefore, with probability at least $1-2e^{-c_2 p}$, $\hat{e} = e^*$ and $\hat{x} = x^*$.
\end{proof}

Similar result for OMP can be derived by simply replacing the conditions for $k$ and $\zeta$ in Step 1 by $k=K+1$ and $\zeta < \frac{1}{3\sqrt{K}}$.

\subsection{Detailed Experimental Evaluation}
 We show a detailed description of our experiments.
 We consider both real-world data that may not exhibit linear structure and synthetic data that does not always follow a Gaussian distribution.

\begin{figure*}
    \centering
    \subfigure[Gaussian $z$]{\includegraphics[width=0.32\textwidth]{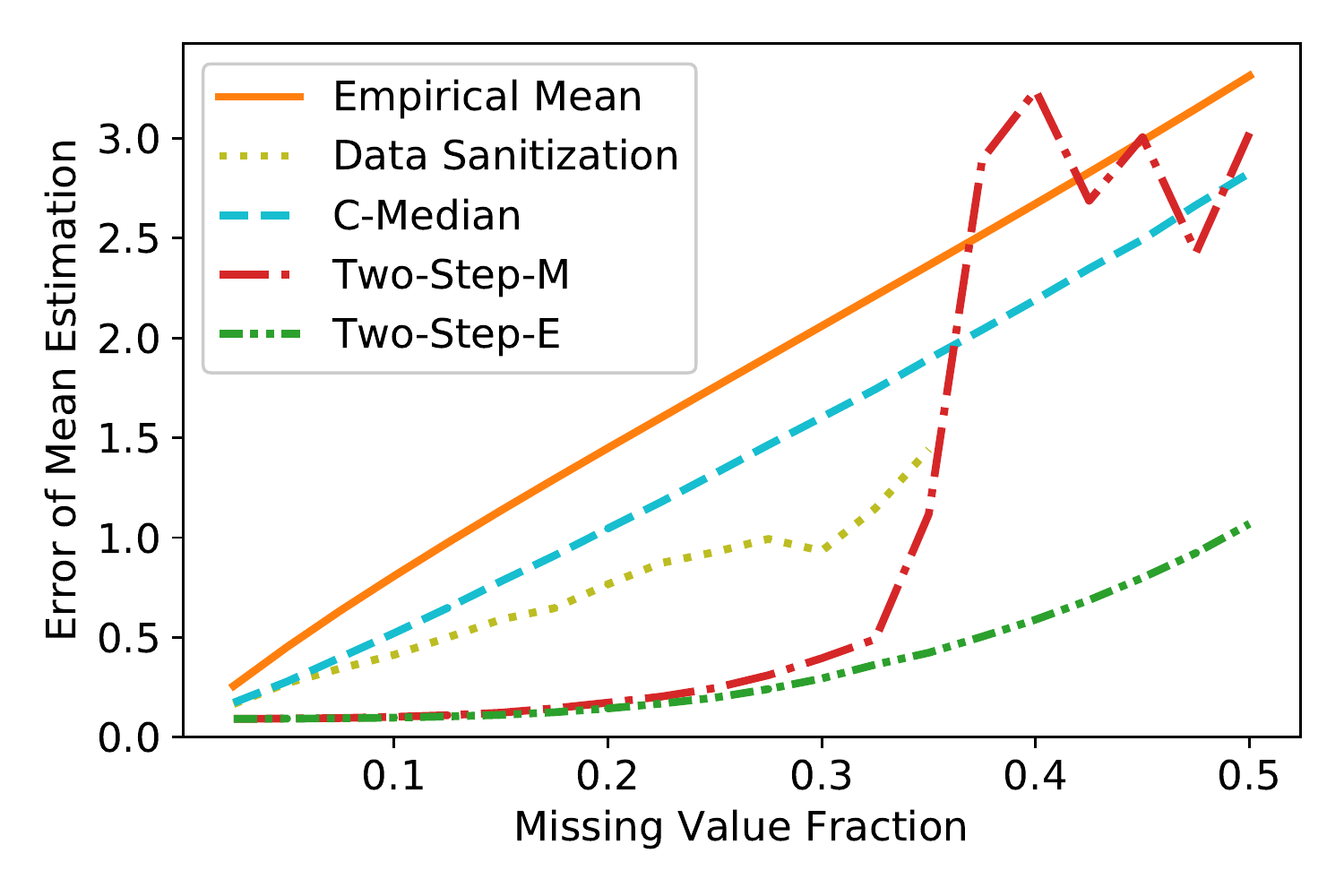}} 
    \subfigure[Uniform $z$]{\includegraphics[width=0.32\textwidth]{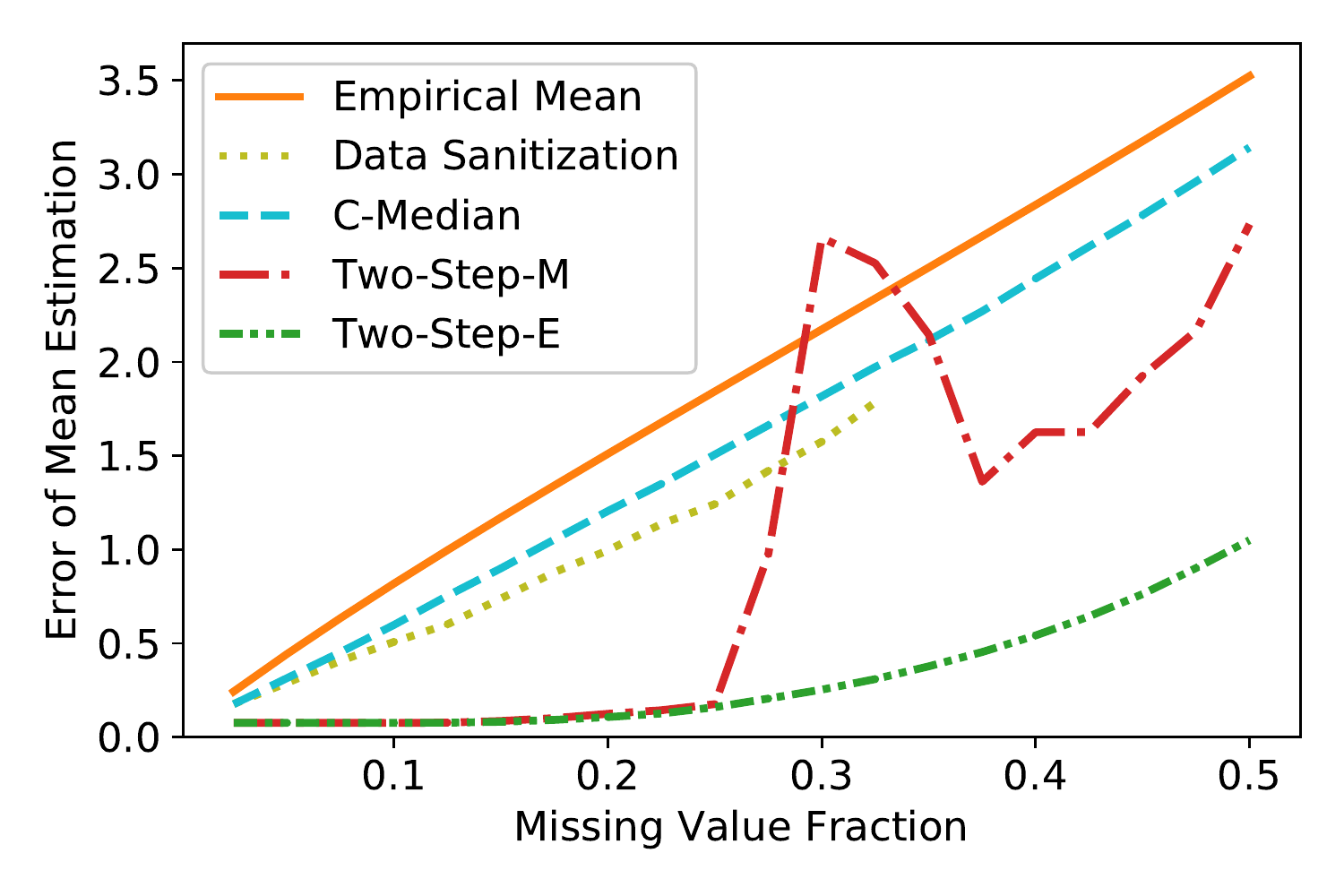}} 
    \subfigure[Exponential $z$]{\includegraphics[width=0.32\textwidth]{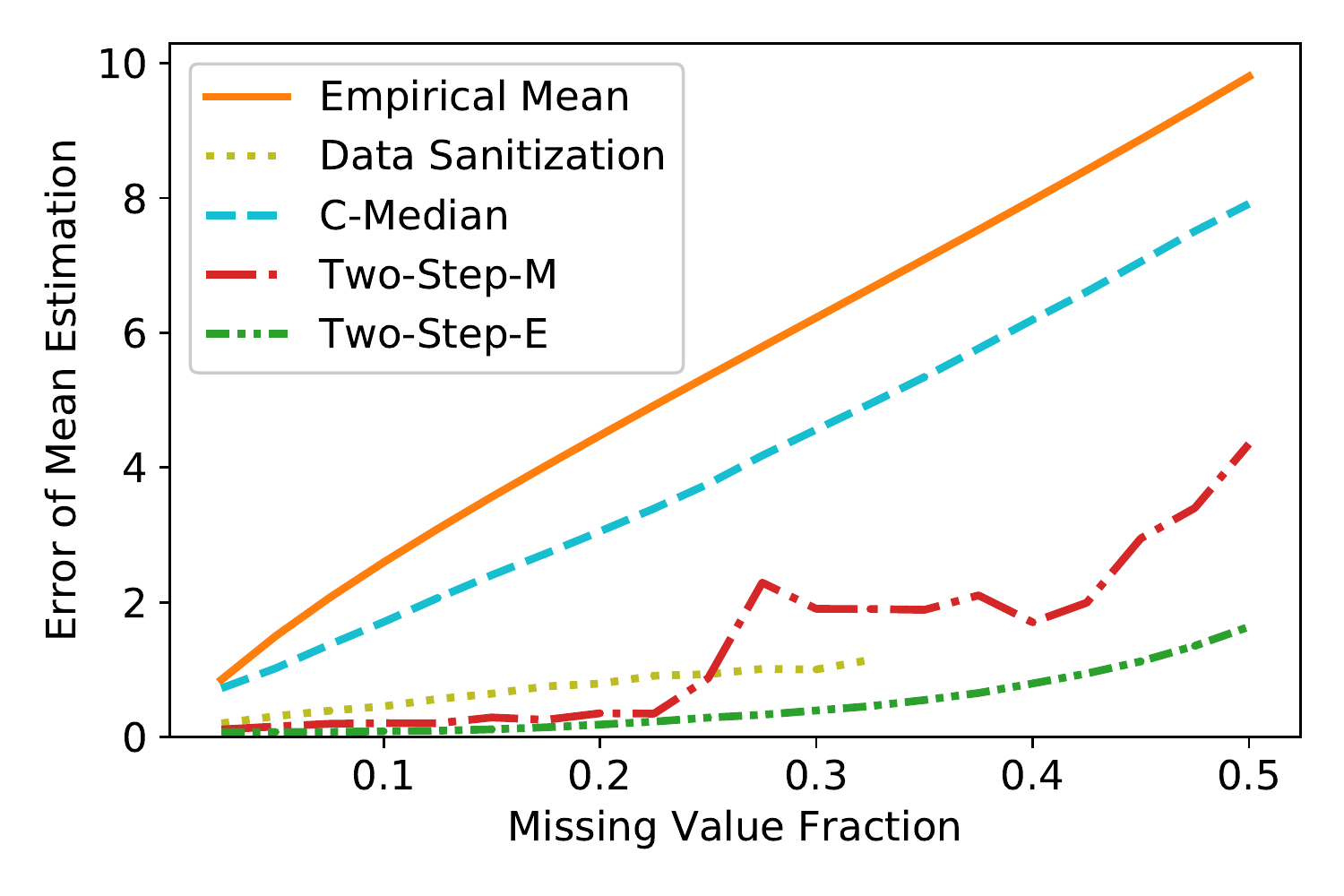}}
    \caption{Mean estimation error (Mahalanobis) for synthetic data sets.}
    \label{fig:syntheticMahalanobis}
\end{figure*}

\begin{figure*}
    \centering
    \subfigure[Gaussian $z$]{\includegraphics[width=0.32\textwidth]{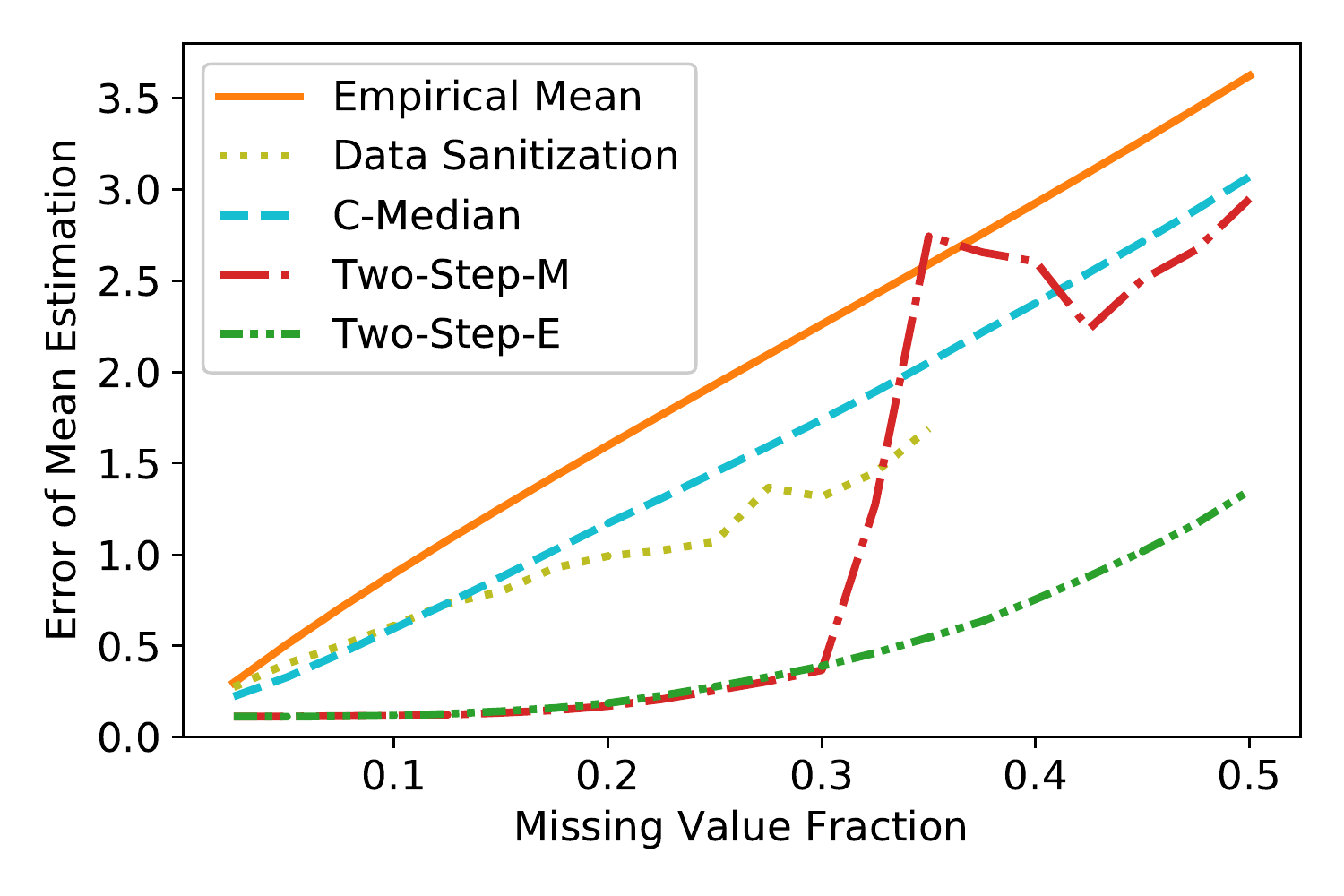}} 
    \subfigure[Uniform $z$]{\includegraphics[width=0.32\textwidth]{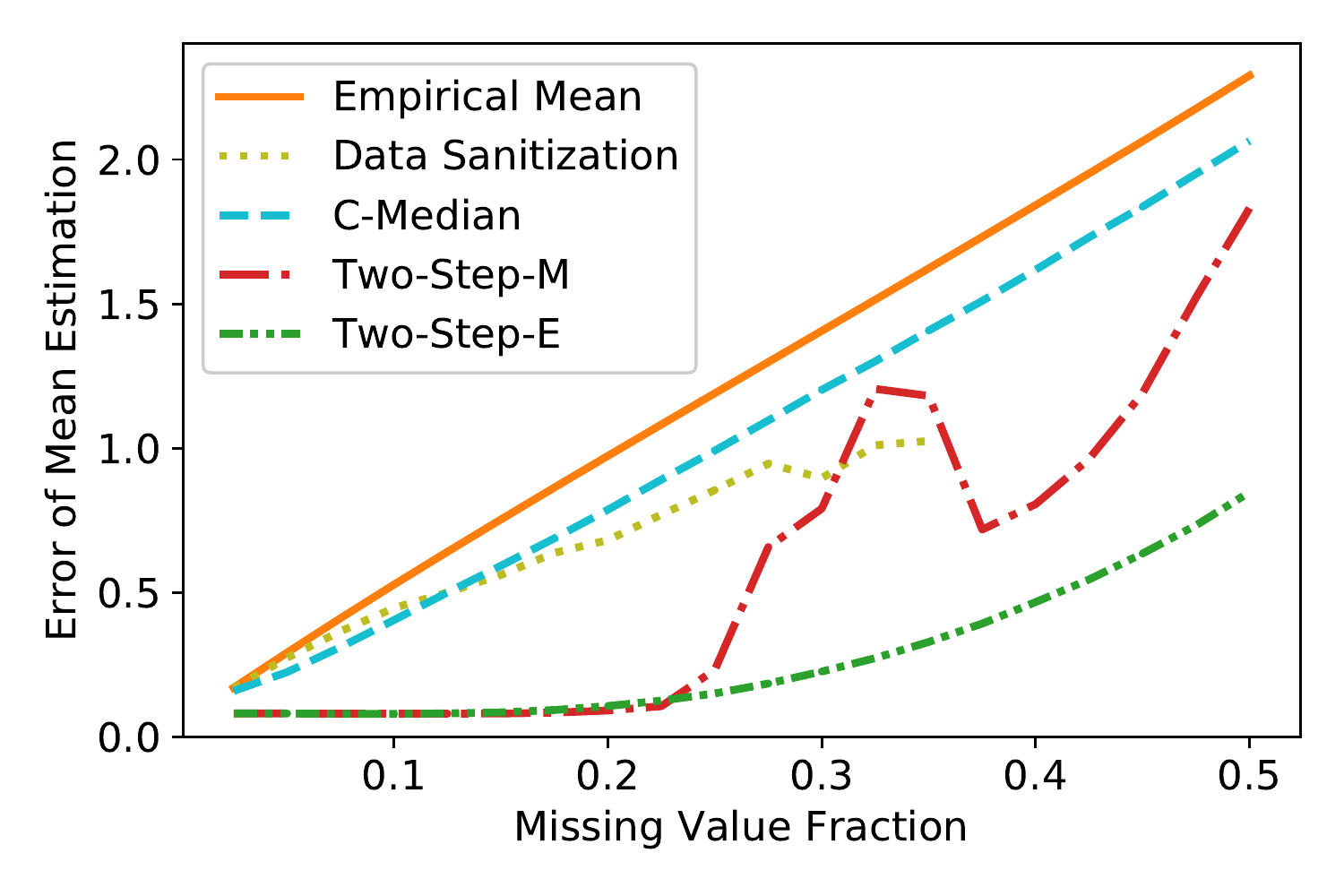}} 
    \subfigure[Exponential $z$]{\includegraphics[width=0.32\textwidth]{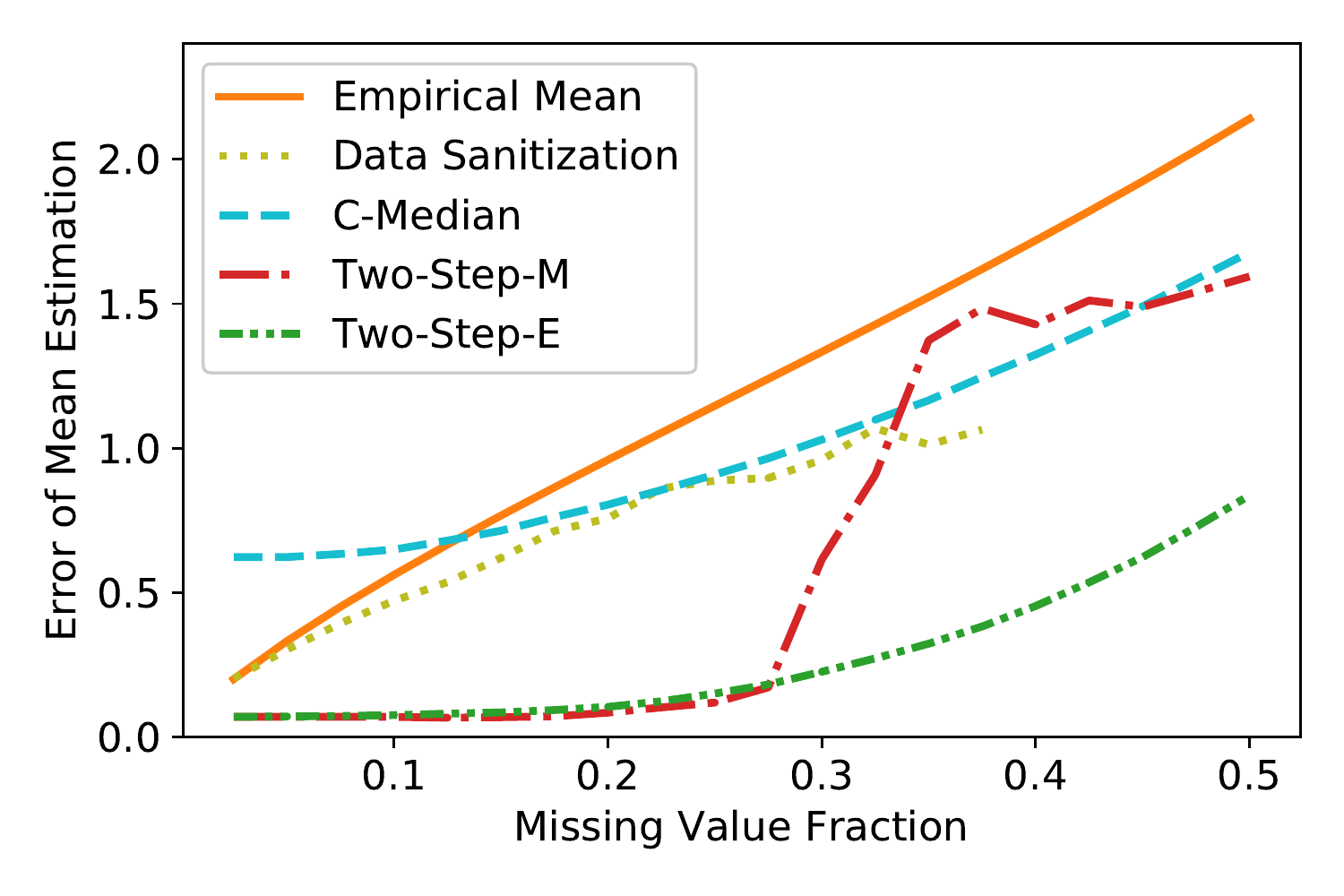}}
    \caption{Mean estimation error (in $l_2$) for synthetic data sets.}
    \label{fig:synthetic}
\end{figure*}

\newparagraph{Methods and Experimental Setup}
We consider the following mean estimation methods:
\begin{itemize}
    \item \textbf{Empirical Mean:} Take the mean for each coordinate, ignoring all missing entries.
    \item \textbf{Data Sanitization:} Remove any samples with missing entries, and then take the mean of the rest of the data.
    \item \textbf{Coordinate-wise Median (C-Median):} Take the median for each coordinate, ignoring all missing entries.
    \item \textbf{Our Method with Matrix Completion (Two-Step-M):} Use iterative hard-thresholded SVD (ITHSVD)~\cite{ithsvd} to impute the missing entries. Take the mean afterwards. We use randomized SVD~\cite{randomSVD} to accelerate.
    \item \textbf{Out Method with Exact Recovery (Two-Step-E):} For each sample, build a linear system based on the structure and solve it. If the linear system is under-determined, do nothing. Then, take the mean while ignoring the remaining missing values.
\end{itemize}

The methods can be classified into three categories, based on the amount of structural information they leverage: (1) Empirical Mean, Data Sanitization, and C-Median ignore the structure information; (2) Two-Step-M assumes there exists some unknown structure but it can be inferred from the visible data; (3) Two-Step-E knows exactly what the structure is and uses it to impute the missing values.

In each experiment presented below, we inject missing values by hiding the smallest $\epsilon$ fraction of each dimension. 
For synthetic data sets, the true mean is derived from the data generation procedure. 
For real-world data sets, we use the empirical mean of the samples before corruption approximate the true mean.
For synthetic data sets, we consider the $l_2$ and Mahalanobis distances to measure the estimation accuracy of different methods.
For real-world data, we only consider the $l_2$ distance between the estimated mean and the true empirical mean of the data before corruption. 

\newparagraph{Mean Estimation on Synthetic Data}
We show that redundancy in the corrupted data can help improve the robustness of mean estimation. 
We test all the methods on synthetic data sets with linear structure ($x = Az$) and three kinds of latent variables ($z$): 1) Gaussian, 2) Uniform, and 3) Exponential. 
Each sample $x_i$ is generated by $x_i = Az_i$, where $z_i$ is sampled from the distribution $D_z$ describing the latent variable $z$. 
We set $A$ to be a diagonal block matrix with two $8 \times 4$ blocks generated randomly and fixed through the experiments. 
In every experiment, we consider a sample with 1,000 data vectors.
To reduce the effect of random fluctuations, we repeat our experiments for five random instances of the latent distribution $D_z$ for each type of latent distribution and take the average error.

The results for the above experiments are shown in Figure~\ref{fig:syntheticMahalanobis}. This figure shows the mean estimation error of different methods measured using the Mahalanobis distance. Additional results with the $l_2$ distance are shown in Figure~\ref{fig:synthetic}.
We see that estimators that leverage the redundancy in the observed data to counteract corruption yield more accurate mean estimates.
This behavior is consistent across all types of distributions and not only for the case of Gaussian distributions that the theoretical analysis in Section~\ref{sec:mean_estimation} focuses on.
We see that the performance of Two-Step-M (when the structure of $A$ is considered unknown) is the same as that of Two-Step-E (when the structure of $A$ is known) when the fraction of missing entries is below a certain threshold.
Following our analysis in Section~\ref{sec:algorithms}, this threshold corresponds to the conditions for which the subspace spanned by the samples can be learned from the visible data.
Finally, we point out that we do not report results for Data Sanitization when the missing fraction is high because all samples get filtered.

\begin{figure*}[t]
    \centering
    \subfigure[Leaf]{\includegraphics[width=0.32\textwidth]{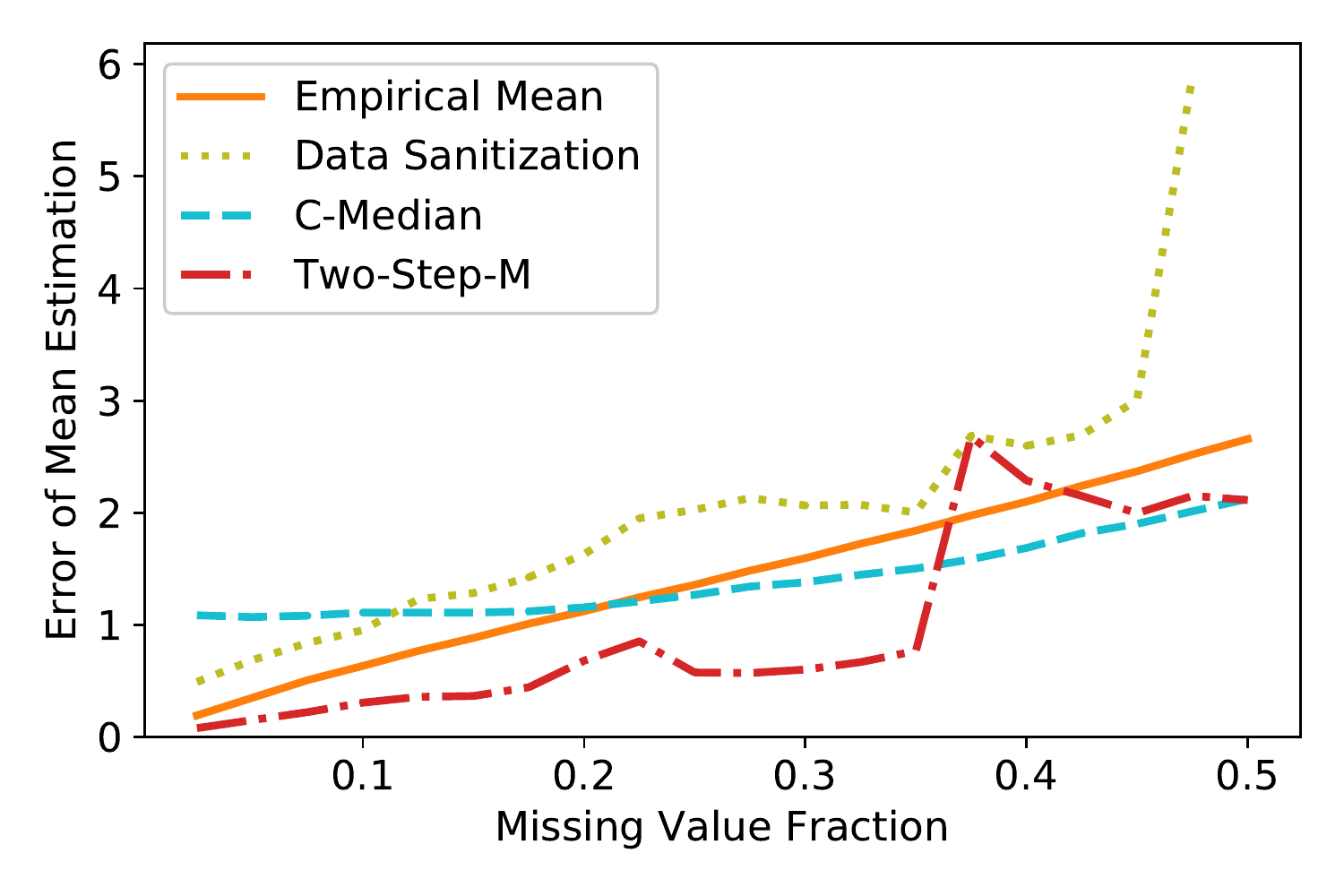}}
    \subfigure[Breast Cancer Wisconsin]{\includegraphics[width=0.32\textwidth]{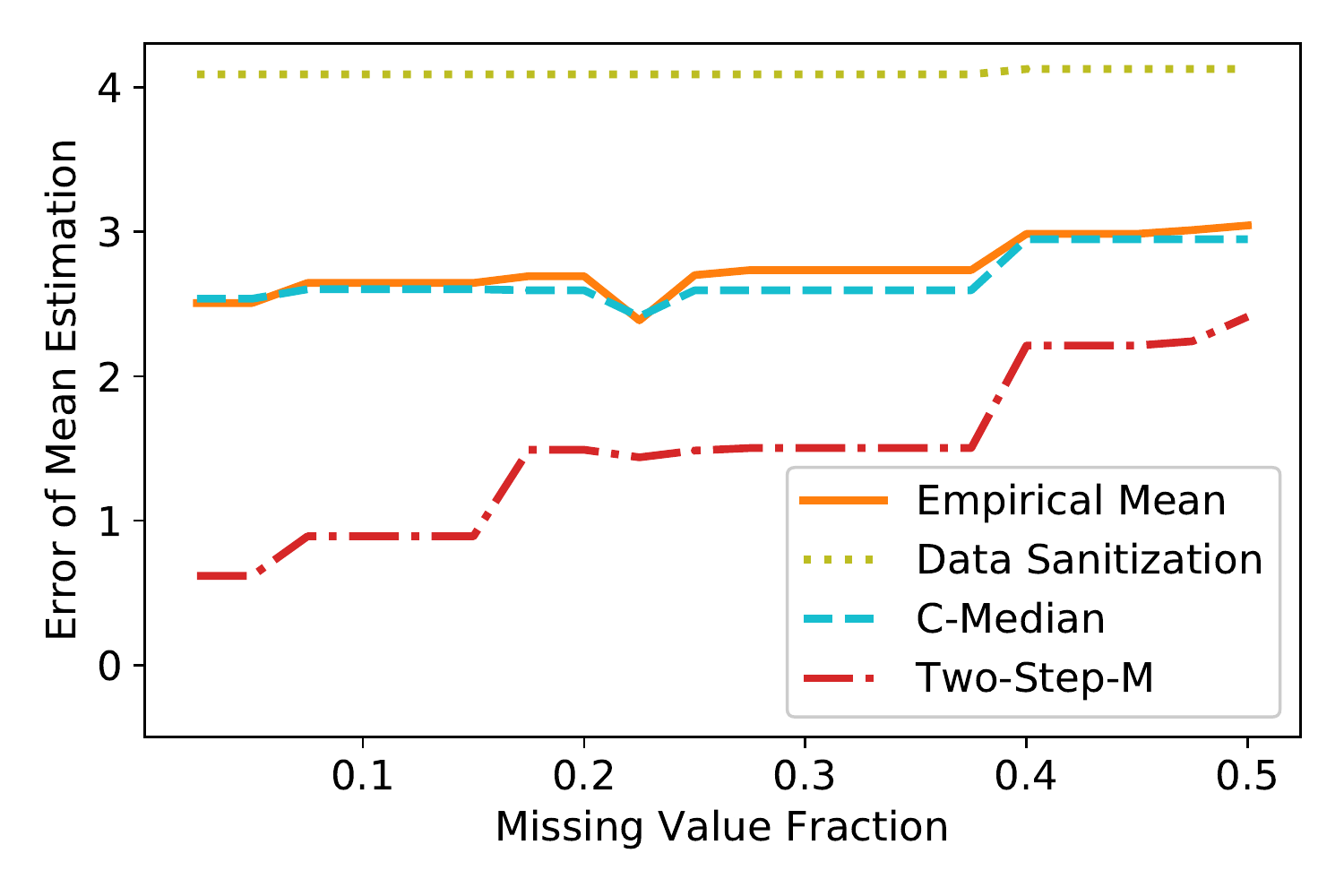}} 
    \subfigure[Blood Transfusion]{\includegraphics[width=0.32\textwidth]{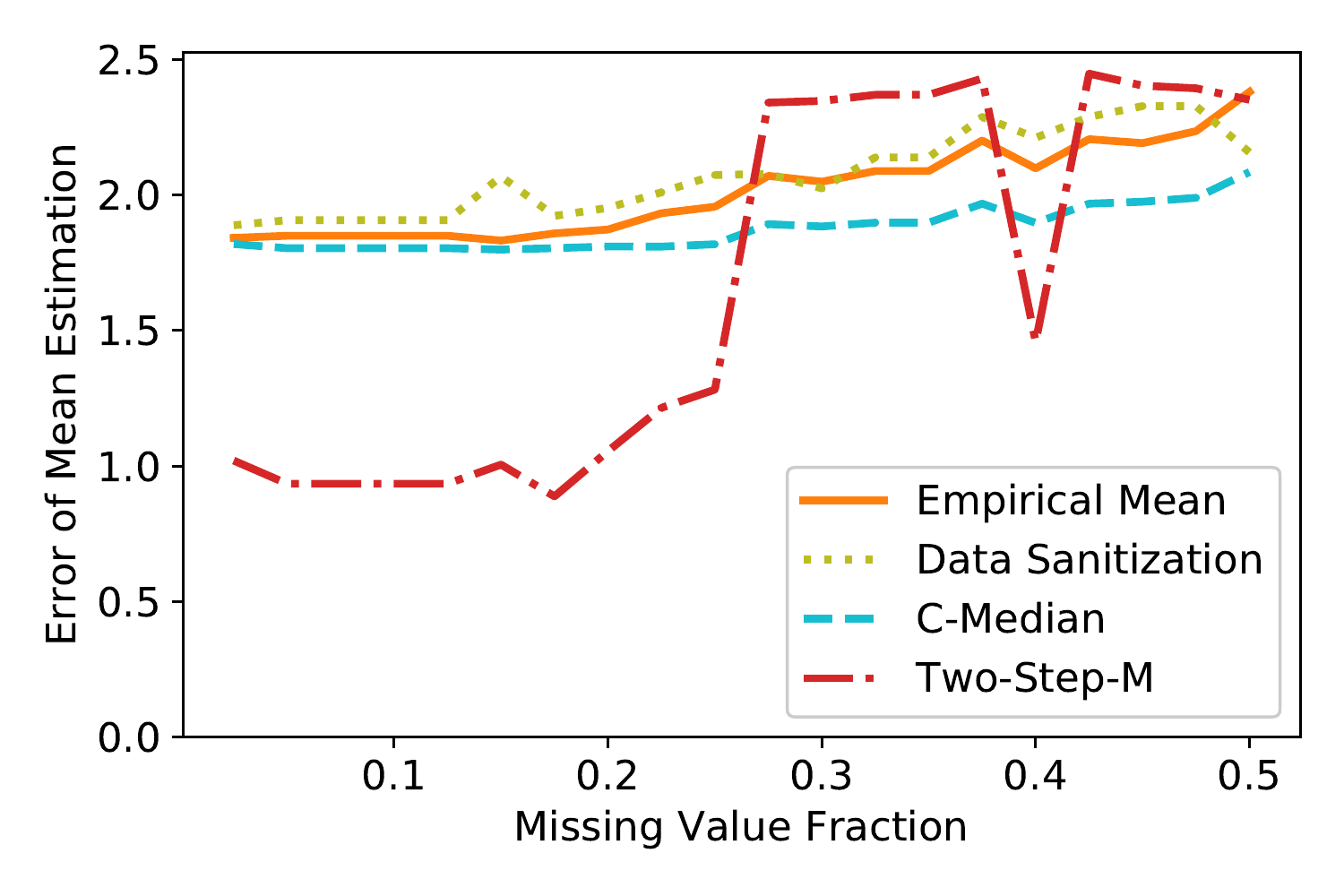}}
    \subfigure[Wearable Sensor]{\includegraphics[width=0.32\textwidth]{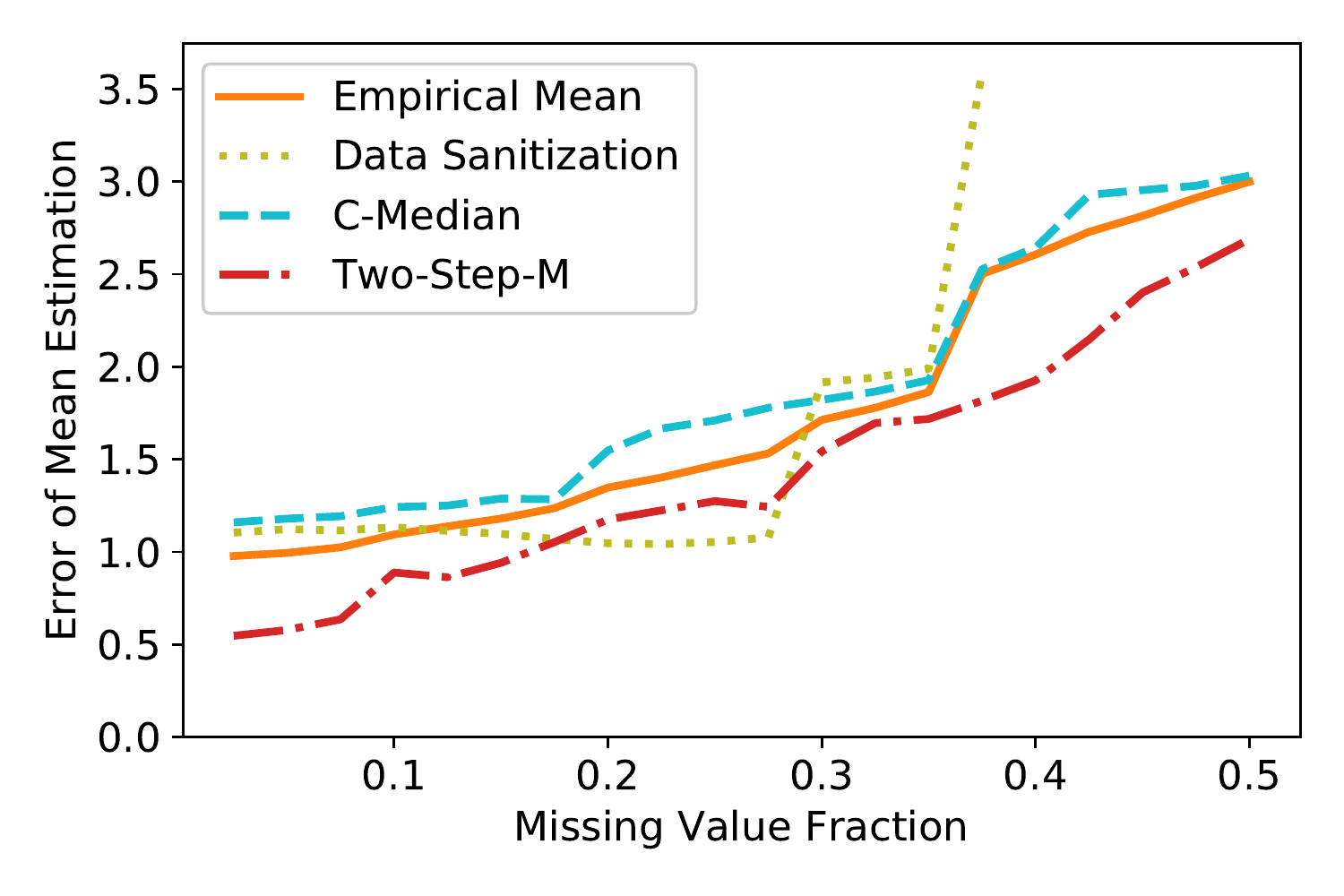}} 
    \subfigure[Mice Protein]{\includegraphics[width=0.32\textwidth]{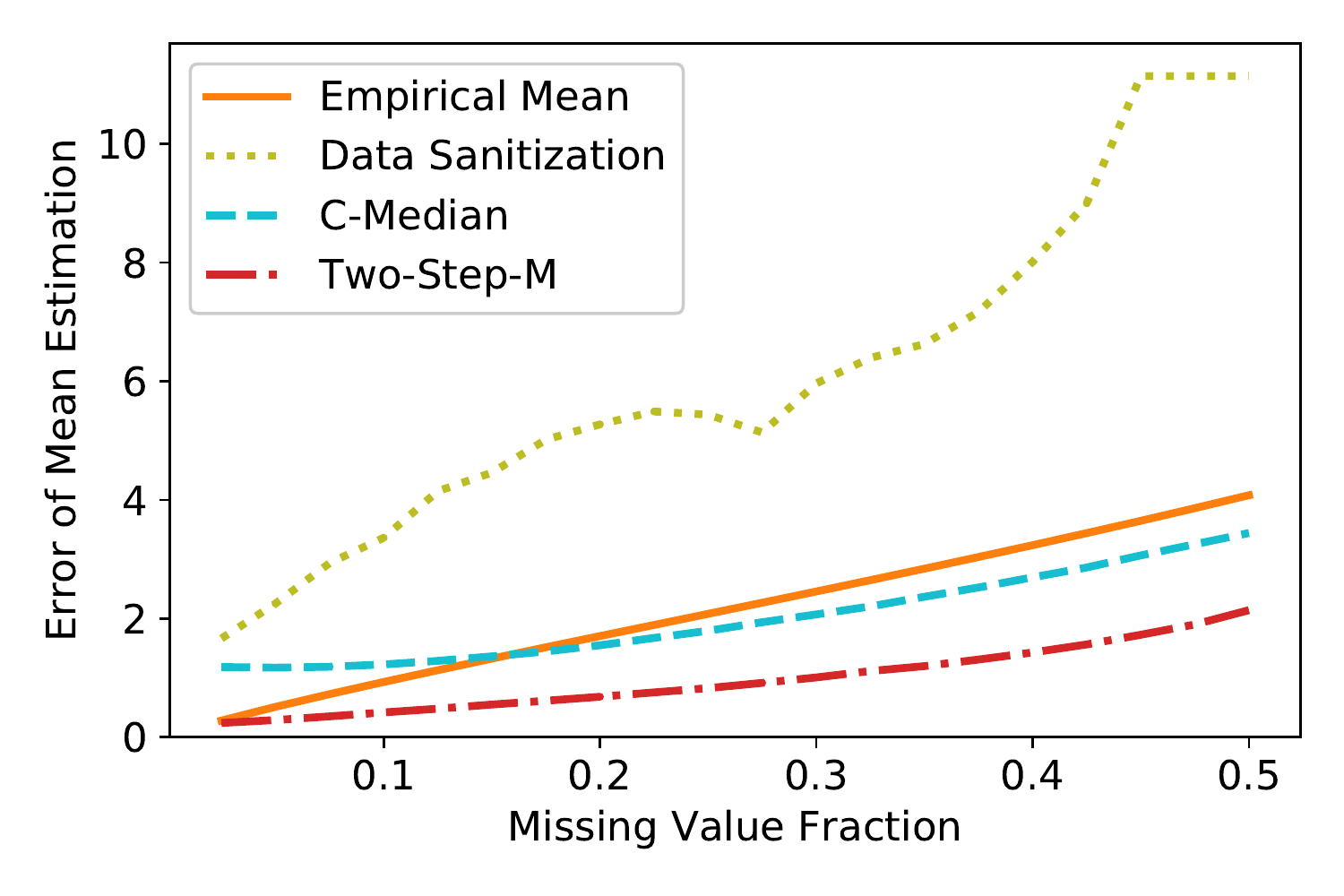}}
    \caption{Error of mean estimation on real-world data sets.}
    \label{fig:realworld_appendix}
\end{figure*}                                                   

\begin{table}[t]
\caption{Properties of the real-world data sets in our experiments.}
\small
\center
\label{tab:dataset}
\vskip 0.15in
\begin{tabular}{lcccr}
\toprule
Data Set & Samples & Features & ITHSVD Rank \\
\midrule
Leaf    & 340 & 14 & 3 \\
Breast Cancer & 69 & 10 & 3\\
Blood Transfusion  & 748 & 5 & 3 \\
Wearable Sensor & 52081 & 9 & 4\\
Mice Protein Expr. & 1080 & 77 & 10 \\
\bottomrule
\end{tabular}
\vskip -0.1in
\end{table}

\newparagraph{Mean Estimation on Real-world Data}
We turn our attention to settings with real-world data with unknown structure. 
We use five data sets from the UCI repository~\cite{UCI} for the experiments in this section.
Specifically, we consider: Leaf \cite{Leaf}, Breast Cancer Wisconsin \cite{Breast}, Blood Transfusion \cite{Transfusion}, Wearable Sensor \cite{Sensor}, and Mice Protein Expression \cite{Mice}. 
For each data set, we consider the numeric features; all of these features are also standardized. 
For all the data sets, we report the $l_2$ error. 
We summarize the size of the data sets along with the rank used for Two-Step-M in Table \ref{tab:dataset}. 
As the structure is unknown, we omit Two-Step-E. 
We show the results in Figure \ref{fig:realworld_appendix}.
We find that Two-Step-M always outperforms Empirical Mean and C-Median on three data sets (Breast Cancer Wisconsin, Wearable Sensor, Mice Protein Expression). 
For the other two (Leaf, Blood Transfusion), the error of Two-Step-M can be as much as two-times lower than the error of the other two methods for small $\epsilon$'s ($\epsilon < 0.35$ for Leaf and $\epsilon < 0.25$ for Blood Transfusion).
We also see that the estimation error becomes very high only for large values of $\epsilon$.
It is also interesting to observe that Data Sanitization performs worse than the Empirical Mean and the C-Median for real-world data.
Recall that the opposite behavior was recorded for the synthetic data setups in the previous section.
Overall, these results demonstrate that structure-aware robust estimators can outperform the standard filtering-based robust mean estimators even in setups that do not follow the linear structure setup in Section~\ref{sec:mean_estimation}.





\end{document}